\documentclass{article} 

\usepackage{url,hyperref,microtype}
\usepackage[onehalfspacing]{setspace}
\usepackage{tikz}
\usepackage{graphicx}
\usepackage{multirow}
\usepackage{amsmath}
\usepackage{mathrsfs} 
\usepackage{amssymb}
\usepackage{amsmath}
\usepackage{makecell}

\usepackage{amsthm}
\usepackage{bbm}
\usepackage[margin=1in]{geometry}
\usepackage{caption}
\usepackage{diagbox}
\usepackage{verbatim}
\allowdisplaybreaks
\usepackage{mathtools}
\usepackage[linesnumbered,ruled,vlined]{algorithm2e}
\usepackage{amsfonts}
\usetikzlibrary{matrix,positioning}
\newcommand{\etal}{\textit{et al.}}
\newtheorem{mythm}{Theorem}
\newtheorem{myprop}[mythm]{Proposition}
\newtheorem{definition}[mythm]{Definition}
\newtheorem{lemma}[mythm]{Lemma}

\newtheorem{remark}{Remark}
\DeclareMathOperator*{\argmin}{arg\,min}
\DeclareMathOperator*{\argmax}{arg\,max}
\DeclareMathOperator{\sign}{sign}

\author{Kevin Bui \thanks{Department of Mathematics;
University of California, Irvine;
Irvine, CA 92697, United States; \url{kevinb3@uci.edu}} \and Yifei Lou \thanks{Department of Mathematical Sciences;
University of Texas, Dallas; Richardson, TX 75080, United States; \url{yifei.lou@utdallas.edu} } \and Fredrick Park \thanks{Department of Mathematics \& Computer Science;
Whittier College; Whittier, CA 90602, United States; \url{fpark@whittier.edu}} \and Jack Xin \thanks{Department of Mathematics;
University of California, Irvine;
Irvine, CA 92697, United States; \url{jxin@math.uci.edu}}}

\begin{document}
\onecolumn

\title{Difference of Anisotropic and Isotropic TV for Segmentation under Blur and Poisson Noise}

\maketitle

\begin{abstract}
In this paper, we aim to segment an image degraded by blur and Poisson noise. We adopt a smoothing-and-thresholding (SaT) segmentation framework that finds a piecewise-smooth solution, followed by $k$-means clustering to segment the image. Specifically for the image smoothing step, we replace the least-squares fidelity for Gaussian noise in the Mumford-Shah model with a maximum posterior (MAP) term to deal with Poisson noise and we incorporate the weighted difference of anisotropic and isotropic total variation (AITV) as a regularization to promote the sparsity of image gradients. For such a nonconvex model, we develop a specific splitting scheme and utilize a proximal operator to apply the alternating direction method of multipliers (ADMM). Convergence analysis is provided to validate the efficacy of the ADMM scheme.   Numerical experiments on various segmentation scenarios (grayscale/color and multiphase) showcase that our proposed method outperforms a number of segmentation methods, including the original SaT. 

\end{abstract}

\section{Introduction}
Image segmentation partitions an image into multiple, coherent regions, where pixels of one region share similar characteristics such as colors, textures, and edges. It remains an important yet challenging  problem in computer vision that has various applications, including magnetic resonance imaging \mbox{\cite{duan2015l_, li2022image, tongbram2021novel}} and microscopy \cite{bui2020segmentation, zosso2017image}.  One of the most fundamental models for image segmentation is the Mumford-Shah model \cite{mumford1989optimal} because of its robustness to noise.  Given an input image $f: \Omega \rightarrow \mathbb{R}$ defined on an open, bounded, and connected domain $\Omega \subset \mathbb{R}^2$, the Mumford-Shah model is formulated as
\begin{align} \label{eq:MS_model}
    \min_{u, \Gamma} E_{MS}(u, \Gamma) \coloneqq &\frac{\lambda}{2} \int_{\Omega} (f-u)^2 \;dx + \frac{\mu}{2} \int_{\Omega \setminus \Gamma} |\nabla u|^2 \;dx + \text{Length}(\Gamma), 
\end{align}
where $u: \Omega \rightarrow \mathbb{R}$ is a piecewise-smooth approximation of the image $f$, $\Gamma \subset \Omega$ is a compact curve representing the region boundaries, and $\lambda, \mu >0$ are the weight parameters. The first term in \eqref{eq:MS_model} is the fidelity term that ensures that the solution $u$ approximates the image $f$. The second term enforces $u$ to be piecewise smooth on $\Omega \setminus \Gamma$. The last term measures the perimeter, or more mathematically the one-dimensional Haussdorf measure in $\mathbb{R}^2 
$  \cite{bar2011mumford}, of the curve $\Gamma$. However, \eqref{eq:MS_model} is difficult to solve because the unknown set of boundaries needs to be discretized.  One common approach  involves  approximating the objective function in \eqref{eq:MS_model} by a sequence of elliptic functionals \cite{ambrosio1990approximation}. 

Alternatively, Chan and Vese (CV) \cite{chan-vese-2001} simplified \eqref{eq:MS_model} by assuming the solution $u$ to be piecewise constant with two phases or regions, thereby making the model easier to solve via the level-set method \cite{osher1988fronts}. Let the level-set function $\phi$ be Lipschitz continuous and be defined as follows:
\begin{align*}
    \begin{cases}
    \phi(x) > 0 &\text{ if } x \text{ is inside } \Gamma, \\
    \phi(x) = 0 &\text{ if } x \text{ is at } \Gamma, \\
    \phi(x) <0 & \text{ if } x \text{ is outside } \Gamma.
    \end{cases}
\end{align*}
By the definition of $\phi$, the curve $\Gamma$ is represented by $\phi(x) = 0.$
The image region can be defined as either inside or outside the curve $\Gamma$. In short, the CV model is formulated as
\begin{align} \label{eq:CV}
    \min_{c_1, c_2, \phi} E_{CV}(c_1, c_2, \phi) &\coloneqq \lambda \left( \int_{\Omega} |f-c_1|^2 H(\phi) \;dx+  \int_{\Omega} |f-c_2|^2 (1-H(\phi))\;dx \right)+ \nu \int_{\Omega} |\nabla H(\phi)| \;dx,
\end{align}
where $\lambda, \nu$ are weight parameters, the constants $c_1, c_2$ are the mean intensity values of the two regions, and $H(\phi)$ is the Heaviside function defined by $H(\phi) = 1$ if $\phi \geq 0$ and $H(\phi) = 0$ otherwise. A convex relaxation \cite{chan-esedoglu-nikolova-2004} of \eqref{eq:CV} was formulated as 
\begin{align*}
    \min_{c_1, c_2, u \in [0,1]}   \lambda \left( \int_{\Omega} |f-c_1|^2 u \;dx+  \int_{\Omega} |f-c_2|^2 (1-u)\;dx \right)+ \nu \int_{\Omega} |\nabla u| \;dx, 
\end{align*}
where an image segmentation $\tilde{u}$ is obtained by thresholding $u$, that is
\begin{align*}
    \tilde{u}(x) = \begin{cases}
    1 &\text{ if } u(x) > \tau,\\
    0 &\text{ if } u(x) \leq \tau,
    \end{cases}
\end{align*}
for some value $\tau \in (0,1)$. It 
can be solved efficiently by convex optimization algorithms, such as the alternating direction method of multipliers (ADMM) \cite{boyd2011distributed} and primal-dual hybrid gradient \cite{chambolle2011first}. 
 A multiphase extension of \eqref{eq:CV} was proposed in \cite{vese2002multiphase}, but it requires that the number of regions to be segmented is a power of 2. For segmenting into an arbitrary number of regions, fuzzy membership functions were incorporated \cite{li2010multiphase}.

Cai \etal~\cite{cai2013two} proposed the smoothing-and-thresholding (SaT) framework that is related to the model \eqref{eq:MS_model}. In the smoothing step of SaT, a convex variant of \eqref{eq:MS_model} is formulated as
\begin{align}\label{eq:convex_MS}
   u^*= \arg\min_u \frac{\lambda}{2} \int_{\Omega} (f-Au)^2 \;dx + \frac{\mu}{2} \int_{\Omega} |\nabla u|^2 \;dx + \int_{\Omega} |\nabla u | \;dx,
\end{align}
yielding a piecewise-smooth solution $u^*$. The blurring operator $A$ is included in the case when the image $f$ is blurred. The total variation (TV) term $\int_{\Omega} |\nabla u| \;dx$ is a convex approximation of the length term in \eqref{eq:CV} by the coarea formula \cite{chan-esedoglu-nikolova-2004}. After the smoothing step, a thresholding step is applied to the smooth image $u^*$ to segment it into multiple regions. The two-stage framework has many advantages. First, the smoothing model \eqref{eq:convex_MS} is strongly convex, so it can be solved by any convex optimization algorithm to obtain a unique solution $u^*$. Second,  the user can adjust the number of thresholds to segment $u^*$ and the threshold values to obtain a satisfactory segmentation result, 
thanks to the flexibility of the thresholding step. Furthermore, the SaT framework  can be adapted to color images by incorporating an intermediate lifting step \cite{cai2017three}. Before performing the thresholding step, the lifting step converts the RGB space to Lab (perceived lightness, red-
green and yellow-blue) color space and concatenates both RGB and Lab intensity values into a six-channel image. The multi-stage framework for color image segmentation is called smoothing, lifting, and thresholding (SLaT). 

One limitation of \eqref{eq:convex_MS} lies in the $\ell_2$ fidelity term that is statistically designed for images corrupted by additive Gaussian noise, and as a result, the smoothing step is not applicable to other types of noise distribution. In this paper, we aim at Poisson noise, which is commonly encountered when an image is taken by photon-capturing devices such as in positron emission tomography \cite{vardi1985statistical} and astronomical imaging \cite{lanteri2005restoration}. By using the data fidelity term of $Au - f \log Au$ \cite{le2007variational}, we obtain a smoothing model that is appropriate for Poisson noise \cite{chan2014two}:
\begin{align}\label{eq:Poisson_MS}
        \min_u \lambda \int_{\Omega} (Au - f \log Au )\;dx + \frac{\mu}{2} \int_{\Omega} |\nabla u|^2 \;dx + \int_{\Omega} |\nabla u | \;dx.
\end{align}

As a convex approximation of the length term in \eqref{eq:MS_model}, the TV term in \eqref{eq:Poisson_MS} can be further improved by nonconvex regularizations. The TV regularization is defined by the $\ell_1$ norm of the image gradient. Literature  has shown that nonconvex regularizations often yield better performance than the convex $\ell_1$ norm   in identifying sparse solutions. Examples of nonconvex regularization include $\ell_p, 0 < p < 1,$ \cite{cao2013fast, chartrand2007exact, xu2012l_}, $\ell_1 - \alpha \ell_2, \alpha \in (0,1]$ \cite{ding2019regularization,ge2021dantzig, li2020, lou2015computational,lou-2015-cs}, $\ell_1/\ell_2 $ \cite{rahimi2019scale,wang2020accelerated,xu2021analysis}, and an error function \cite{guo2021novel}. 
Lou \etal~\cite{lou2015weighted} designed a TV version of $\ell_1 - \alpha \ell_2$ called the weighted difference of anisotropic--isotropic total variation (AITV), which outperforms TV in various imaging applications, such as image denoising \cite{lou2015weighted}, image reconstruction \cite{lou2015weighted, li2020}, and image segmentation \cite{bui2022smoothing, bui2021weighted, wu2022image}.

In this paper, we propose an AITV variant of \eqref{eq:Poisson_MS} to improve the smoothing step of the SaT/SLaT framework for images degraded by Poisson noise and/or blur. Incorporating AITV regularization is motivated by our previous works \cite{ bui2022smoothing, bui2021weighted,park2016weighted}, where we demonstrated that AITV regularization  is effective in preserving edges and details, especially under Gaussian and impulsive noise. To maintain similar computational efficiency as the original SaT/SLaT framework, we propose an ADMM algorithm that utilizes the $\ell_1-\alpha \ell_2$ proximal operator \cite{lou2018fast}. The main contributions of this paper are as follows:
\begin{itemize}
    \item We propose an AITV-regularized variant of \eqref{eq:Poisson_MS} and prove the existence of a minimizer for the model.
    \item We develop a computationally efficient ADMM algorithm and provide its convergence analysis under certain conditions.
    \item We conduct numerical experiments on various grayscale/color images to demonstrate the effectiveness of the proposed approach.
\end{itemize}

The rest of the paper is organized as follows. Section \ref{sec:prelim} describes the background information such as notations, Poisson noise, and the SaT/SLaT framework. In Section \ref{sec:proposed_method}, we propose a simplified Mumford-Shah model with AITV and a MAP data fidelity term for Poisson noise. In the same section, we show that the model has a global minimizer and develop an ADMM algorithm with convergence analysis. In Section \ref{sec:experiment}, we evaluate the performance of the AITV Poisson SaT/SLaT framework on various grayscale and color images. Lastly, we conclude the paper in Section \ref{sec:conclude}. 
\section{Preliminaries} \label{sec:prelim}
\subsection{Notation}
Throughout the rest of the paper, we  represent images and mathematical models in discrete notations (i.e., vectors and matrices). An image is represented as an $M \times N$ matrix, and hence the image domain is denoted by $\Omega = \{1, 2, \ldots, M\} \times \{1,2, \ldots, N\}$. We define two inner product spaces: $X \coloneqq \mathbb{R}^{M \times N}$ and $Y \coloneqq X \times X$. Let $u \in X$. For shorthand notation, we define $u \geq 0$ if $u_{i,j} \geq 0$ for all $(i,j) \in \Omega$. The discrete gradient operator $\nabla: X \rightarrow Y$ is defined by $(\nabla u)_{i,j} = \left((\nabla_x u)_{i,j},
    (\nabla_y u)_{i,j} \right)$,
where
	\begin{align*}
	(\nabla_x u)_{i,j} = \begin{cases}
	u_{i,j} - u_{i,j-1} &\text{ if } 2 \leq j \leq N, \\
	u_{i,1} - u_{i,N} &\text{ if } j = 1,
	\end{cases}
\mbox{\ and\ }
	(\nabla_y u)_{i,j} = \begin{cases}
	u_{i,j} - u_{i-1,j} &\text{ if } 2 \leq i \leq M, \\
	u_{1,j} - u_{M,j} &\text{ if } i = 1.
	\end{cases}
	\end{align*}
The space $X$ is equipped with the standard inner product $\langle \cdot, \cdot \rangle_X$ and  Euclidean norm $\|\cdot\|_2$. The space $Y$ has the following inner product and norms: for $p=(p_1, p_2) \in Y$ and $q=(q_1, q_2) \in Y$, 
\begin{align*}
    \langle p, q \rangle_Y &= \langle p_1, q_1 \rangle_X + \langle p_2, q_2 \rangle_X,\\
    \|p\|_1 &= \sum_{i=1}^M \sum_{j=1}^N |(p_1)_{i,j}|+|(p_2)_{i,j}|, \\
    \|p\|_2 & = \sqrt{\sum_{i=1}^M \sum_{j=1}^N|(p_1)_{i,j}|^2+|(p_2)_{i,j}|^2} , \\
    \|p\|_{2,1} &= \sum_{i=1}^M \sum_{j=1}^N \sqrt{(p_1)_{i,j}^2 + (p_2)_{i,j}^2}. 
\end{align*}
For brevity, we omit the subscript $X$ or $Y$ in the inner product when its context is clear.

\subsection{AITV Regularization}
There are two popular discretizations of total variation: the isotropic TV \cite{rudin1992nonlinear} and the anisotropic TV \cite{choksi-2011}, which are defined by 
\begin{align*}
   & \|\nabla u\|_{2,1} = \sum_{i=1}^M \sum_{j=1}^N \sqrt{|(\nabla_x u)_{i,j}|^2 + |(\nabla_y u)_{i,j}|^2}, \\
  &  \|\nabla u\|_1 = \sum_{i=1}^M \sum_{j=1}^N |(\nabla_x u)_{i,j}| + |(\nabla_y u)_{i,j}|,
\end{align*} 
respectively.
This work is based on the weighted difference between anisotropic and isotropic TV (AITV)  regularization~\cite{lou2015weighted}, defined by
\begin{align} 
    \|\nabla u\|_1 - \alpha\|\nabla u\|_{2,1} = \sum_{i=1}^M \sum_{j=1}^N \left(|(\nabla_x u)_{i,j}| + |(\nabla_y u)_{i,j}| -  \alpha \sqrt{|(\nabla_x u)_{i,j}|^2 + |(\nabla_y u)_{i,j}|^2}\right),
\end{align}
for a weighting parameter $\alpha \in [0,1].$ The range of $\alpha$ ensures the non-negativity of the AITV regularization. 
Note that anisotropic TV is defined as the $\ell_1$ norm of the image gradient $((\nabla_x u)_{i,j}, (\nabla_y u)_{i,j})$ at the pixel location $(i,j) \in \Omega$,  while  isotropic TV is the $\ell_2$ norm on the gradient vector. As a result, AITV can be viewed as the $\ell_1 - \alpha \ell_2$ regularization on the gradient vector at every pixel, thereby enforcing sparsity individually at each gradient vector. 

\subsection{Poisson Noise}
Poisson noise follows the Poisson distribution with mean and variance $\eta$, whose probability mass function is given by
\begin{align}\label{eq:poi-dist}
\mathbb{P}_{\eta}(n) = \frac{e^{-\eta} \eta^n}{n!}, \; n \geq 0.
\end{align}
For a clean image $g \in X$, its intensity value at each pixel $g_{i,j}$ serves as the mean and variance for the corresponding noisy observation $f \in X$ defined by
\begin{align*}
    f_{i,j} \sim \text{Poisson}(g_{i,j})\; \forall (i,j) \in \Omega. 
\end{align*}
To recover the image $g$ from the noisy image $f$, we find its maximum a posteriori (MAP) estimation $u$, which maximizes the probability $\mathbb{P}(u|f)$. By Bayes' theorem, we have
\begin{align*}
    \mathbb{P}(u|f) = \frac{\mathbb{P}(f|u)\mathbb{P}(u)}{\mathbb{P}(f)}.
\end{align*}
It further follows from the definition \eqref{eq:poi-dist} that
\begin{align*}
    \mathbb{P}(f_{i,j}|u_{i,j})\mathbb{P}(u_{i,j}) = \mathbb{P}_{u_{i,j}}(f_{i,j}) \mathbb{P}(u_{i,j}) = \frac{e^{-u_{i,j}} u_{i,j}^{f_{i,j}}}{(f_{i,j})!}\mathbb{P}(u_{i,j}) .
\end{align*}
Since Poisson noise is i.i.d. pixelwise, we have
\begin{align*}
    \mathbb{P}(u|f) = \prod_{(i,j)\in \Omega} \mathbb{P}(f_{i,j}|u_{i,j}) \frac{\mathbb{P}(u_{i,j})}{\mathbb{P}(f_{i,j})} = \prod_{(i,j) \in \Omega} \frac{e^{-u_{i,j}} u_{i,j}^{f_{i,j}}}{(f_{i,j})!}\frac{\mathbb{P}(u_{i,j})}{\mathbb{P}(f_{i,j})}.
\end{align*}
The MAP estimate of $\mathbb{P}(u|f)$ is equivalent to its negative logarithm, thus leading to the following optimization problem:
\begin{align}
    \min_{u \geq 0 } \sum_{(i,j) \in \Omega} u_{i,j} - f_{i,j} \log u_{i,j} - \log \mathbb{P}(u_{i,j}). 
\end{align}
The last term $- \log \mathbb{P}(u_{i,j})$ can be regarded as an image prior or a regularization. For example, Le \etal~ \cite{le2007variational} considered the isotropic total variation as the image prior and proposed a Poisson denoising model 
\begin{align} \label{eq:Poisson_denoise}
    \min_{u \geq 0} \langle u - f \log u, \mathbbm{1} \rangle + \|\nabla u\|_{2,1},
\end{align}
where $\log$ is applied pixelwise and $\mathbbm{1}$ is the matrix whose entries are all 1's.  The first term in \eqref{eq:Poisson_denoise} is 
a concise notation that is commonly used as a fidelity term for Poisson denoising in various imaging applications~\cite{chan2014two, chang2018total, chowdhury2020non,rahman2020poisson,le2007variational, wen2016primal}. 

\subsection{Review of Poisson SaT/SLaT}
A Poisson SaT framework \cite{chan2014two} consists of two   steps. Given a noisy grayscale image $f \in X$ corrupted by Poisson noise, the first step is the smoothing step that finds a piecewise-smooth solution $u^*$ from the optimization model:
\begin{align}\label{eq:ms_poisson_discrete}
    u^* = \argmin_{u \geq 0}  \lambda \langle Au - f \log Au, \mathbbm{1} \rangle + \frac{\mu}{2} \|\nabla u\|_2^2 + \|\nabla u\|_{2,1}. 
\end{align}
Then in the thresholding step, $K-1$ threshold values $\tau_1 \leq \tau_2 \leq \ldots \leq \tau_{K-1}$ are appropriately chosen to segment $u^*$ into $K$ regions, where the $k$th region is given by 
\begin{align*}
    \Omega_{k} = \{(i,j) \in \Omega: \tau_{k-1} \leq u^*_{i,j} < \tau_k\},
\end{align*}
with $\tau_0 \coloneqq \inf_{x \in \Omega} u^*(x)$. The thresholding step is typically performed by $k$-means clustering.

The Poisson smoothing, lifting, and thresholding (SLaT) framework \cite{cai2017three} extends the Poisson SaT framework to color images. For a color image $f = (f_1, f_2, f_3) \in X \times X \times X$, the model
\eqref{eq:ms_poisson_discrete} is applied to each color channel $f_i$ for $i=1,2,3$, thus leading to a  smoothed color image $u^* = (u_1^*, u_2^*, u_3^*)$. An additional lifting step \cite{luong1993color} is performed to transform $u^*$ to $(u_1', u_2', u_3')$ in the Lab space (perceived lightness, red-green, and yellow-blue). The channels in Lab space are less correlated than in RGB space, so they may have useful information for segmentation. The RGB image and the Lab image are concatenated to form the multichannel image $\hat{u} \coloneqq (u_1^*, u_2^*, u_3^*, u_1', u_2', u_3')$, followed by  the thresholding stage. 
 Generally, $k$-means clustering yields $K$ centroids $c_1, \ldots, c_K$ as constant vectors, which are used to form the region
\begin{align*}
    \Omega_{k} =
    \left \{(i,j) \in \Omega : \| \hat{u}_{i,j} -c_{k}\|_2 = \min_{1 \leq \kappa \leq K} \|\hat{u}_{i,j} - c_{\kappa}\|_2 \right\}
\end{align*}
for $k=1, \ldots, K$ such that $\Omega_{k}$'s are disjoint and $\bigcup_{k=1}^K \Omega_{k} = \Omega$. 

After the thresholding step for both SaT/SLaT, we define a piecewise-constant approximation of the image $f$ by
\begin{align}\label{eq:pc_constant_approx}
    \tilde{f} = (\tilde{f}_1, \ldots, \tilde{f}_d) \text{ such that } \tilde{f}_{\ell} = \sum_{k=1}^K c_{k,\ell} \mathbbm{1}_{\Omega_{k}}\; \forall \ell=1, \ldots, d,
\end{align}
where $c_{k,\ell}$ is the $\ell$th entry of the constant vector $c_k$ and 
\begin{align*}
\mathbbm{1}_{\Omega_{k}} = \begin{cases}
    1 &\text{ if } (i,j) \in \Omega_{k}, \\
    0 &\text{ if } (i,j) \not \in \Omega_{k}.
    \end{cases} 
\end{align*}
Recall that $d=1$ when $f$ is grayscale, and $d=3$ when $f$ is color.
\section{Proposed Approach}\label{sec:proposed_method}
To improve the Poisson SaT/SLaT framework, we propose to replace the isotropic TV in \eqref{eq:ms_poisson_discrete} with AITV regularization. In other words, in the smoothing step, we obtain the smoothed image $u^*$ from the optimization problem
\begin{align}\label{eq:ms_poisson_aitv}
        u^* = \argmin_{u} F(u) \coloneqq  \lambda \langle Au - f \log Au, \mathbbm{1} \rangle + \frac{\mu}{2} \|\nabla u\|_2^2 + \|\nabla u\|_1 - \alpha \|\nabla u\|_{2,1},
\end{align}
for $\alpha \in [0,1]$. We establish that this model admits a global solution. We then develop an ADMM algorithm to find   a solution and provide the convergence analysis. The overall segmentation approach is described in Algorithm \ref{alg:sat_slat}.

\begin{algorithm}[t]
  \textbf{Input}:{\begin{itemize}
      \item image $f= (f_1, \ldots, f_d)$
      \item blurring operator $A$
      \item fidelity parameter $\lambda >0$
      \item smoothing parameter $\mu \geq 0$
      \item AITV parameter $\alpha \in [0,1]$
      \item the number of regions in the image $K$
  \end{itemize}}
  \textbf{Output}: {Segmentation $\tilde{f}$}\\
   Stage one: Compute $u_{\ell}$ by solving \eqref{eq:ms_poisson_aitv} separately for $\ell = 1, \ldots, d$.\\
   Stage two:
   \If{$f$ is a grayscale image, i.e., $d=1$}{Go to stage three. }
   \ElseIf{$f$ is a color image, i.e., $d=3$}{Transfer the solution $u^*=(u_1^*, u_2^*,u_3^*)$ into Lab space to obtain $(u_1', u_2', u_3')$ and concatenate to form $\hat{u}=(u_1^*, u_2^*,u_3^*, u_1', u_2', u_3')$.}
   Stage three: Apply $k$-means to obtain $\{(c_{k}, \Omega_{k})\}_{k=1}^K$ and compute $\tilde{f}$ by \eqref{eq:pc_constant_approx}.
\caption{AITV Poisson SaT/SLaT}
\label{alg:sat_slat}
\end{algorithm}

\subsection{Model Analysis}
 To establish the solution's existence of the proposed model \eqref{eq:ms_poisson_aitv}, we  start with Lemma~\ref{lemma:poincare}, a discrete version of Poincar\'e's inequality \cite{evans2010partial}. In addition, we prove Lemma~\ref{lemma: some_ineq1} and Proposition~\ref{prop:bound}, leading to the global existence theorem (Theorem~\ref{thm:solution_exist}).

\begin{lemma}\label{lemma:poincare} 
There exists a constant $C>0$ such that 
\begin{align} \label{eq:discrete_poincare}
   \left \|u - \bar{u} \mathbbm{1}\right \|_2 \leq C \|\nabla u\|_{2,1},
\end{align}
for every $u \in X$ and $\bar{u} \coloneqq \displaystyle \frac{1}{MN} \sum_{i=1}^{M} \sum_{j=1}^N u_{i,j}$.
\end{lemma}
\begin{proof}
We prove it by contradiction. Suppose there exists a sequence $\{u_k\}_{k=1}^{\infty}$ such that
\begin{align}\label{eq:contradict_ineq}
    \left \|u_k - \bar{u}_k \mathbbm{1} \right \|_2 > k \|\nabla u_k\|_{2,1},
\end{align}
where $\bar{u}_k = \displaystyle \frac{1}{MN} \sum_{i=1}^{M}\sum_{j=1}^N (u_k)_{i,j}$.  For every $k$, we normalize each element in the sequence by
$
    v_k = \frac{u_k - \bar{u}_k \mathbbm{1} }{\left\|u_k - \bar{u}_k\mathbbm{1} \right\|_2}.
$
It is straightforward that
\begin{align} \label{eq:sequence_prop}
    \bar{v}_k = \frac{1}{MN} \sum_{i=1}^M \sum_{j=1}^N (v_k)_{i,j} = 0, \quad \|v_k\|_2 = 1\quad \forall k\in \mathbb{N}.
\end{align}
By \eqref{eq:contradict_ineq}, we have
\begin{align}\label{eq:gradient_seq}
\|\nabla v_k\|_{2,1} < \frac{1}{k}.    
\end{align}
As $\{v_k\}_{k=1}^{\infty}$ is bounded, there exists a convergent subsequence $\{v_{k_j}\}_{j=1}^{\infty}$ such that $v_{k_j} \rightarrow v^*$ for  $v^* \in X$. It follows from \eqref{eq:gradient_seq} that $
    \|\nabla v^*\|_{2,1}  = 0.$ Since $\text{ker}(\nabla) = \{c \mathbbm{1}: c \in \mathbb{R}\}$,  then $v^*$ is a constant vector. However, \eqref{eq:sequence_prop} implies that $\bar{v}^* = 0$ and $\|v^*\|_2=1$. This contradiction proves the lemma.
\end{proof}
\begin{lemma} \label{lemma: some_ineq1}
    Suppose $\|f\|_{\infty} < \infty$ and $\min_{i,j} f_{i,j} > 0$. There exists a scalar $u_0 > 0$ such that  we have $2(x - f_{i,j} \log x ) \geq x$ for any $\ x \geq u _0$ and $(i,j) \in \Omega$. 
\end{lemma}
\begin{proof}
For each $(i,j) \in \Omega$, we want to show that there exists $u_{i,j} > 0$ such that $H(x) \coloneqq x - 2 f_{i,j} \log x \geq 0$ for $x \geq u_{i,j}$. Since $H(x)$ is strictly convex and it attains a global minimum at $x = 2f_{i,j}$, it is increasing on the domain $x > 2f_{i,j}$. Additionally as $x$ dominates $\log(x)$ as $x\rightarrow +\infty$, 
there exists $u_{i,j} > 2 f_{i,j} > 0$ such that 
$\frac{u_{i,j}}{\log u_{i,j}} \geq 2 f_{i,j},$ 
which implies that $H(u_{i,j}) = u_{i,j} - 2f_{i,j} \log u_{i,j} \geq 0.$
As a result, for $x \geq u_{i,j} > 2f_{i,j}$, we obtain
$
x - 2f_{i,j} \log x = H(x) \geq H(u_{i,j})\geq 0.
$ Define $u_0 \coloneqq \max_{i,j} u_{i,j}$, and hence we have $2(x - f_{i,j} \log x) \geq x$ for  $x \geq u_0 \geq u_{i,j}, \forall (i,j)\in \Omega$. 
\end{proof}

\begin{myprop}\label{prop:bound}
Suppose $\text{ker}(A) \cap \text{ker}(\nabla) = \{0\}$ and $\{u_k\}_{k=1}^{\infty} \subset X$. If $\{(Au_k, \nabla u_k)\}_{k=1}^{\infty}$ is bounded, then $\{u_k\}_{k=1}^{\infty}$ is bounded.
\end{myprop}
\begin{proof}
Since $\text{ker}(A) \cap \text{ker}(\nabla) = \{0\}$, we have $A\mathbbm{1} \neq 0$.
Simple calculations lead to
\begin{gather}
\begin{aligned}\label{eq:ineq2}
    |\bar{u}_k|\|A \mathbbm{1} \|_2 &= \|A (\bar{u}_k \mathbbm{1}) \|_2 \leq \| A (\bar{u}_k \mathbbm{1} - u_k)\|_2 + \|Au_k\|_2\\ & \leq \|A\| \|u_k - \bar{u}_k \mathbbm{1}\|_2 + \|Au_k\|_2\\ &\leq C\|A\| \|\nabla u_k\|_{2,1} + \|Au_k\|_2,
\end{aligned}
\end{gather}
where the last inequality is due to  Lemma \ref{lemma:poincare}.
The boundedness of $\{Au_k\}_{k=1}^{\infty}$ and $\{\nabla u_k\}_{k=1}^{\infty}$ implies that $\{\bar{u}_k\}_{k=1}^{\infty}$ is also bounded by \eqref{eq:ineq2}. We apply Lemma \ref{lemma:poincare} to obtain
\begin{align*}
    \|u_k\|_2 \leq \|u_k - \bar{u}_k\mathbbm{1}\|_2 + \|\bar{u}_k \mathbbm{1}\|_2 < C \|\nabla u_k \|_{2,1} + \|\bar{u}_k \mathbbm{1}\|_2 < \infty,
\end{align*}
which thereby proves that $\{u_k\}_{k=1}^{\infty}$ is bounded.
\end{proof}

Finally, we adapt the proof in \cite{chan2014two} to establish that $F$ has a global minimizer.

\begin{mythm}\label{thm:solution_exist}Suppose $\|f\|_{\infty} < \infty$ and $\min_{i,j} f_{i,j} > 0$. 
If $\lambda > 0, \mu \geq 0, \alpha \in [0,1)$, and $\text{ker}(A) \cap \text{ker}(\nabla) = \{0\}$, then $F$ has a global minimizer. 
\end{mythm}
\begin{proof}
It is straightforward that  $\|\nabla u\|_{2,1} \leq \|\nabla u\|_1$, thus $\|\nabla u\|_1 - \alpha \|\nabla u\|_{2,1}\geq 0$ for $\alpha\in[0,1)$. As a result, we have
\begin{align*}
    F(u) \geq \lambda  \langle Au - f \log Au, \mathbbm{1} \rangle = \lambda \sum_{i=1}^M \sum_{j=1}^N (Au)_{i,j} - f_{i,j} \log (Au)_{i,j}.
\end{align*}

Given a scalar $f>0,$ the function $G(x)=x-f\log(x)$ attains its global minimum at $x=f.$ Therefore, 
we have $x - f_{i,j} \log x \geq f_{i,j} - f_{i,j} \log f_{i,j}$  for all $x > 0$ and $(i,j) \in \Omega$, which leads to a lower bound of $F(u)$, i.e.,
\begin{align}\label{eq:ineq1}
    F(u) \geq \lambda \sum_{i=1}^M \sum_{j=1}^N (Au)_{i,j} - f_{i,j} \log (Au)_{i,j} \geq \lambda \sum_{i=1}^M \sum_{j=1}^N f_{i,j} - f_{i,j} \log f_{i,j} \eqqcolon F_{0}.
\end{align}

As $F(u)$ is lower bounded by  $F_0$,  we can choose a minimizing sequence $\{u_k\}_{k=1}^{\infty}$ and hence $F(u_k)$ has a uniform upper bound, denoted by $B_1$, i.e., $F(u_k) < B_1$ for all $k \in \mathbb{N}$.  It further follows from \eqref{eq:ineq1} that
\begin{align*}
    B_1 \geq F(u_k) \geq \lambda  \langle Au_k - f \log Au_k, \mathbbm{1} \rangle \geq F_0,
\end{align*}
which implies that $\{|\langle Au_k - f \log Au_k, \mathbbm{1} \rangle|\}_{k=1}^{\infty}$ is uniformly bounded, i.e., there exists a constant $B_2 > 0$ such that $\left|\langle Au_k - f \log Au_k, \mathbbm{1} \rangle \right| < B_2, \ \forall k$. Using these uniform bounds, we derive that
\begin{align*}
    (1- \alpha) \|\nabla u_k\|_1 \leq \frac{\mu}{2} \|\nabla u_k\|_2^2 + \|\nabla u_k\|_1 - \alpha  \|\nabla u_k\|_{2,1} = F(u_k) -  \lambda  \langle Au_k - f \log Au_k, \mathbbm{1} \rangle \leq B_1+\lambda B_2.
\end{align*}
As $\alpha<1$, the sequence $\{ \nabla u_k\}_{k=1}^{\infty}$ is bounded.

To prove the boundedness of  $\{Au_k\}_{k=1}^{\infty}$, we introduce the notations of   $x^+ = \max(x,0)$ and $x^- = -\min(x,0)$ for any $x \in \mathbb{R}$. Then $x = x^+ - x^-$. 
By Lemma \ref{lemma: some_ineq1}, there exists $u_0 > 0$ such that $2(x - f_{i,j} \log x) \geq x, \ \forall x \geq u_0$ and $(i,j) \in \Omega$. We observe that
\begin{gather}
\begin{aligned} \label{eq:au_bound}
    \|Au_k\|_1 &= \sum_{i=1}^M \sum_{j=1}^N |(Au_k)_{i,j}|  \leq \sum_{i=1}^M \sum_{j=1}^N \max \{ 2 ((Au_k)_{i,j} - f_{i,j} \log (Au_k)_{i,j}), u_0 \} \\
    &\leq 2  \sum_{i=1}^M \sum_{j=1}^N \left((Au_k)_{i,j} - f_{i,j} \log (Au_k)_{i,j} \right)^+ + MNu_0 \\
    & = 2 \sum_{i=1}^M \sum_{j=1}^N \left[\left((Au_k)_{i,j} - f_{i,j} \log (Au_k)_{i,j} \right) + \left((Au_k)_{i,j} - f_{i,j} \log (Au_k)_{i,j} \right)^- \right]+ MNu_0 \\
    &= 2\langle Au_k - f \log Au_k, \mathbbm{1} \rangle + 2 \sum_{i=1}^M \sum_{j=1}^N \left((Au_k)_{i,j} - f_{i,j} \log (Au_k)_{i,j} \right)^- +MNu_0\\
    &\leq 2B_2 + 2 \sum_{i=1}^M \sum_{j=1}^N \left| f_{i,j} - f_{i,j} \log f_{i,j}\right| + MNu_0 < \infty.
\end{aligned}
\end{gather}
This shows that $\{Au_k\}_{k=1}^{\infty}$ is bounded.

Since both $\{ \nabla u_k\}_{k=1}^{\infty}$ and $\{Au_k\}_{k=1}^{\infty}$ are bounded, then  $\{u_k\}_{k=1}^{\infty}$ is bounded due to Proposition \ref{prop:bound}. Therefore,  there exists a subsequence $\{u_{k_n}\}_{n=1}^{\infty}$ that converges to some $u^* \in X$. As $F$ is continuous and thus lower semicontinuous, we have
\begin{align*}
    F(u^*) \leq \liminf_{n \rightarrow \infty} F(u_{k_n}), 
\end{align*}
which means that $u^*$ minimizes $F$.  
\end{proof}

\subsection{Numerical Algorithm}

To minimize \eqref{eq:ms_poisson_aitv}, we introduce two auxiliary variables $v \in X$ and $w = (w_x, w_y) \in Y$, leading to an equivalent constrained optimization problem:
    \begin{equation}
\begin{aligned}\label{eq:constrained_opt}
    \min_{u,v, w}  & \quad  \lambda \langle v - f \log v, \mathbbm{1} \rangle + \frac{\mu}{2} \|\nabla u\|_2^2 + \|w \|_1 - \alpha  \|w\|_{2,1} \\
    \text{s.t.} & \quad Au = v,  \quad \nabla u = w.
\end{aligned}
\end{equation}
The corresponding augmented Lagrangian is expressed as
\begin{gather}
\begin{aligned}
\label{eq:lagrange}
    \mathcal{L}_{\beta_1, \beta_2}(u, v,w, y, z) =  &\lambda \langle v - f \log v, \mathbbm{1} \rangle + \frac{\mu}{2} \|\nabla u\|_2^2 + \|w \|_1 - \alpha  \|w\|_{2,1}\\ &+ \langle y, Au -v \rangle + \frac{\beta_1}{2} \|Au - v\|_2^2 + \langle z, \nabla u -w \rangle + \frac{\beta_2}{2} \|\nabla u - w\|_2^2,
    \end{aligned}
    \end{gather}
where $y \in X$ and $z = (z_x, z_y) \in Y$ are Lagrange multipliers and $\beta_1, \beta_2$ are positive parameters.  We then apply the alternating direction method of multipliers (ADMM) to minimize \eqref{eq:constrained_opt} that consists of the following steps per iteration $k$:
\begin{subequations}\label{eq:poisson_admm}
    \begin{align}
    u_{k+1} &= \argmin_u \mathcal{L}_{\beta_{1,k}, \beta_{2,k}} (u, v_k, w_k, y_k, z_k), \label{eq:u_subprob}\\
    v_{k+1} &= \argmin_v \mathcal{L}_{\beta_{1,k}, \beta_{2,k}} (u_{k+1}, v, w_k, y_k, z_k), \label{eq:v_subprob}\\
    w_{k+1} &= \argmin_w \mathcal{L}_{\beta_{1,k}, \beta_{2,k}} (u_{k+1}, v_{k+1}, w, y_k, z_k), \label{eq:w_subprob}\\
    y_{k+1} &= y_k + \beta_{1,k}(Au_{k+1} - v_{k+1}), \label{eq:y_eq}\\
    z_{k+1} &= z_k + \beta_{2,k}(\nabla u_{k+1} - w_{k+1}), \label{eq:z_eq}\\
    (\beta_{1,k+1}, \beta_{2,k+1})  &= \sigma (\beta_{1,k}, \beta_{2,k}), \label{eq:poisson_admm_beta}
    \end{align}
\end{subequations}
where $\sigma > 1$. 
\begin{remark}
The scheme presented in \eqref{eq:poisson_admm} slightly differs from the original ADMM  \cite{boyd2011distributed}, the latter of which has $\sigma = 1$ in \eqref{eq:poisson_admm_beta}. Having $\sigma >1$ increases the weights of the penalty parameters $\beta_{1,k}, \beta_{2,k}$ in each iteration $k$, thus accelerating the numerical convergence speed of the proposed ADMM algorithm. A similar technique has been used in \cite{cascarano2021efficient,gu2017weighted, storath2014fast,storath2014jump,you2019nonconvex}.
\end{remark}

All the subproblems \eqref{eq:u_subprob}-\eqref{eq:w_subprob} have closed-form solutions. In particular, the first-order optimality condition for \eqref{eq:u_subprob} is
\begin{align} \label{eq:first_order_u}
    [\beta_{1,k} A^{\top}A - (\mu+ \beta_{2,k})\Delta]u^{k+1} = A^{\top}(\beta_{1,k}v_k - y_k) - \nabla^{\top}(z_k - \beta_{2,k}w_k),
\end{align}
where $\Delta = -\nabla^{\top}\nabla$ is the Laplacian operator. If $\text{ker}(A) \cap \text{ker}(\nabla) = \{0\}$, then $\beta_{1,k} A^{\top}A - (\mu+ \beta_{2,k})\Delta$ is positive definite and thereby invertible, which implies that \eqref{eq:first_order_u} has a unique solution $u^{k+1}$. By assuming periodic boundary condition for $u$, the operators $\Delta$ and $A^{\top}A$ are block circulant \cite{wang2008new}, and hence \eqref{eq:first_order_u} can be solved efficiently by the 2D discrete Fourier transform $\mathcal{F}$. Specifically, we have the formula
\begin{align}
    u_{k+1} = \mathcal{F}^{-1}\left( \frac{\mathcal{F}(A)^*\circ\mathcal{F}(\beta_{1,k}v_k -y_k) - \mathcal{F}(\nabla)^* \circ \mathcal{F}(z_k - \beta_{2,k} w_k)}{\beta_{1,k} \mathcal{F}(A)^* \circ \mathcal{F}(A) - (\mu + \beta_{2,k} )\mathcal{F}(\Delta)} \right),
\end{align}
where $\mathcal{F}^{-1}$ is the inverse discrete Fourier transform, the superscript $*$ denotes complex conjugate, the operation $\circ$ is componentwise multiplication, and division is componentwise. 
By differentiating the objective function of \eqref{eq:v_subprob} and setting it to zero, we can get a  closed-form solution for $v_{k+1}$ given by
\begin{align}
    v_{k+1} = \frac{\left(\beta_{1,k} Au_{k+1} + y_{k}- \lambda \mathbbm{1}\right)+\sqrt{\left(\beta_{1,k} Au_{k+1} + y_{k}- \lambda \mathbbm{1}\right)^2+4\lambda\beta_{1,k} f}}{2\beta_{1,k}},
\end{align}
where the square root, squaring, and division are performed componentwise. Lastly,  the $w$-subproblem \eqref{eq:w_subprob} can be decomposed componentwise as follows:
\begin{gather}
\begin{aligned}\label{eq:pixel_w_subprob}
    (w_{i,j})_{k+1} &= \argmin_{w_{i,j}}  \|w_{i,j}\|_1 - \alpha \|w_{i,j}\|_2  + \frac{\beta_{2,k}}{2}  \left\| w_{i,j} - \left( (\nabla u_{k+1})_{i,j} + \frac{(z_k)_{i,j}}{\beta_{2,k}}\right) \right\|_2^2\\
    &= \text{prox}\left((\nabla u_{k+1})_{i,j} + \frac{(z_k)_{i,j}}{\beta_{2,k}}, \alpha, \frac{1}{\beta_{2,k}}\right),
\end{aligned}
\end{gather}
where the proximal operator for $\ell_1 - \alpha \ell_2$ on $x \in \mathbb{R}^n$ is given by
\begin{align}\label{eq:prox}
    \text{prox}(x, \alpha, \beta) = \argmin_y \|y\|_1 - \alpha \|y\|_2 + \frac{1}{2 \beta} \|x-y\|_2^2. 
\end{align}
The proximal operator for $\ell_1 - \alpha \ell_2$ has a closed form solution summarized by Lemma \ref{lemma:prox}.

\begin{lemma}[\cite{lou2018fast}]\label{lemma:prox}
Given $x \in \mathbb{R}^n$, $\beta >0$, and $\alpha \in [0,1]$, the optimal solution to \eqref{eq:prox} is given by one of the following cases:
\begin{enumerate}
    \item When $\|x\|_{\infty} > \beta$, we have
    \begin{align*}
        x^* = (\|\xi\|_2 + \alpha \beta) \frac{\xi}{\|\xi\|_2},
    \end{align*}
     where $\xi= \sign(x)\circ\max(|x|-\beta,0)$. 
    \item When $(1-\alpha) \beta < \|x\|_{\infty} \leq \beta$, then $x^*$ is a 1-sparse vector such that one chooses $i \in \displaystyle \argmax_j(|x_j|)$ and defines $x^*_i=\left(|x_i| + (\alpha-1)\beta\right)\sign(x_i)$ and the remaining  elements equal to 0.
\item When $\|x\|_{\infty} \leq (1- \alpha)\beta$, then $x^* = 0$. 
\end{enumerate}
\end{lemma}

In summary, we describe the ADMM scheme to solve \eqref{eq:ms_poisson_aitv} in Algorithm \ref{alg:admm}.
\begin{algorithm*}[t!!!]
  \textbf{Input:}{\begin{itemize}
      \item image $f$
      \item blurring operator $A$
      \item fidelity parameter $\lambda >0$
      \item smoothing parameter $\mu \geq 0$
      \item AITV parameter $\alpha \in [0,1]$
      \item penalty parameters $\beta_{1,0}, \beta_{2,0}> 0$
      \item penalty multiplier $\sigma > 1$
      \item relative error $\epsilon >0 $ 
  \end{itemize}}
  \textbf{Output:}{$u_{k}$}\\
    Initialize $u_0, w_0, z_0$.\\
    Set $k=0$.\\
   \While{$\frac{\|u_{k}-u_{k-1}\|_2}{\|u_{k}\|_2} > \epsilon$}
   {
\begin{subequations}
   \begin{align*}
            u_{k+1} &= \mathcal{F}^{-1}\left( \frac{\mathcal{F}(A)^*\circ\mathcal{F}(\beta_{1,k}v_k -y_k) - \mathcal{F}(\nabla)^* \circ \mathcal{F}(z_k - \beta_{2,k} w_k)}{\beta_{1,k} \mathcal{F}(A)^* \circ \mathcal{F}(A) - (\mu + \beta_{2,k} )\mathcal{F}(\Delta)} \right)\\
           v_{k+1} &= \frac{\left(\beta_{1,k} Au_{k+1} + y_{k}- \lambda \mathbbm{1}\right)+\sqrt{\left(\beta_{1,k} Au_{k+1} + y_{k}- \lambda \mathbbm{1}\right)^2+4\lambda\beta_{1,k} f}}{2\beta_{1,k}}\\
               (w_{k+1})_{i,j} &= \text{prox}\left((\nabla u_{k+1})_{i,j} + \frac{(z_k)_{i,j}}{\beta_{2,k}}, \alpha, \frac{1}{\beta_{2,k}}\right) \quad \forall (i,j) \in \Omega\\
                    y_{k+1} &= y_k + \beta_{1,k}(Au_{k+1} - v_{k+1}) \\
                    z_{k+1} &= z_k + \beta_{2,k}(\nabla u_{k+1} - w_{k+1})\\
    (\beta_{1,k+1}, \beta_{2,k+1})  &= \sigma (\beta_{1,k}, \beta_{2,k})\\
    k &\coloneqq k+1
   \end{align*}
\end{subequations}
   }

\caption{ADMM for the AITV-Regularized Smoothing Model with Poisson Fidelity \eqref{eq:ms_poisson_aitv}}
\label{alg:admm}

\end{algorithm*}
\subsection{Convergence Analysis}
We establish the subsequential convergence of ADMM   described in Algorithm \ref{alg:admm}. The global convergence of ADMM \cite{wang2019global} is inapplicable to our model as the gradient operator $\nabla$ is non-surjective, which will be further investigated in future work. For the sake of brevity, we set $\beta = \beta_1 = \beta_2$ and denote
\begin{align*}
    \mathcal{L}_{\beta}(u,v,w,y,z) \coloneqq \mathcal{L}_{\beta, \beta}(u,v,w,y,z).
\end{align*}
In addition, we introduce definitions of subdifferentials~\cite{rockafellar2009variational}, which defines a stationary point of a non-smooth objective function.


\begin{definition}
For a proper function $h: \mathbb{R}^n \rightarrow \mathbb{R} \cup \{+\infty\}$,  define $\text{dom}(h) \coloneqq \{ x\in \mathbb{R}^n: h(x) < +\infty\}$. 
\begin{enumerate}
    \item[(a)] The regular subdifferential at $x \in \text{dom}(h)$ is given by
    \begin{align*}
        \hat{\partial}{h}(x) \coloneqq \left\{w: \liminf_{x' \rightarrow x, x' \neq x} \frac{h(x')-h(x) - \langle w, x'-x \rangle}{\|x'-x\|} \geq 0\right\}.
    \end{align*}
    \item[(b)]  The (limiting) subdifferential at $x \in \text{dom}(h)$ is given by \begin{align*}
        \partial{h}(x) \coloneqq \left\{ w: \exists \, x_k \rightarrow x \text{ and } w_k \in  \hat{\partial}{h}(x_k) \text{ with } w_k \rightarrow w \text{ and } h(x_k) \rightarrow h(x)\right\}.
    \end{align*}
\end{enumerate}
\end{definition}
An important property of the limiting subdifferential is its closedness: for any $(x_k, v_k) \rightarrow (x,v)$ with $v_k \in \partial{h}(x_k)$, if $h(x_k)  \rightarrow h(x)$, then $v \in \partial{h}(x)$.  

\begin{lemma}\label{lemma:sufficient_decrease}
Suppose that $\text{ker}(A) \cap \text{ker}(\nabla) = \{0\}$ and $0 \leq \alpha < 1$. Let $\{(u_k, v_k, w_k, y_k, z_k)\}_{k=1}^{\infty}$ be a sequence generated by Algorithm \ref{alg:admm}. Then we have
\begin{gather}
\begin{aligned}\label{eq:lagrange_ineq}
    &\mathcal{L}_{\beta_{k+1}}(u_{k+1}, v_{k+1}, w_{k+1}, y_{k+1}, z_{k+1}) - \mathcal{L}_{\beta_k}(u_k, v_k, w_k, y_k, z_k)\\ &\leq - \frac{\nu}{2} \|u_{k+1} - u_k\|_2^2- \frac{\beta_0}{2} \|v_{k+1} - v_k\|_2^2 + \frac{1}{\sigma^{k-1} \beta_0}\left(\left\|y_{k+1}-y_k \right\|_2^2 +\left\|z_{k+1}-z_k \right\|_2^2 \right),
\end{aligned}
\end{gather}
for some constant $\nu > 0$.
\end{lemma}
\begin{proof}
If $\text{ker}(A) \cap \text{ker}(\nabla) = \{0\}$, then $\beta_0 A^{\top}A +(\beta_0 + \mu)\nabla^{\top}\nabla$ is positive definite, and hence there exists $\nu > 0$ such that
\begin{align*}
    \beta_k \|Au\|_2^2 +(\beta_k + \mu)\|\nabla u\|_2^2 \geq \beta_0 \|Au\|_2^2 +(\beta_0 + \mu)\|\nabla u\|_2^2 \geq \nu \|u \|_2^2\; \quad \forall k \in \mathbb{N},
\end{align*}
which implies that $\mathcal{L}_{\beta_k}(u, v_k, w_k, y_k, z_k)$ is strongly convex with respect to $u$ with parameter $\nu$. Additionally, 
$\mathcal{L}_{\beta_k}(u_{k+1}, v, w_k, y_k, z_k)$ is strongly convex with respect to $v$ with parameter $\beta_0 \leq \beta_k$.
It follows from \cite[Theorem 5.25]{beck2017first} that we have
\begin{align}\label{eq:ineq3}
    \mathcal{L}_{\beta_k}(u_{k+1}, v_k, w_k, y_k, z_k)- \mathcal{L}_{\beta_k}(u_k, v_k, w_k, y_k, z_k) &\leq - \frac{\nu}{2} \|u_{k+1} - u_k\|_2^2,\\
    \mathcal{L}_{\beta_k}(u_{k+1}, v_{k+1}, w_k, y_k, z_k) - \mathcal{L}_{\beta_k}(u_{k+1}, v_k, w_k, y_k, z_k) &\leq - \frac{\beta_0}{2} \|v_{k+1} - v_k\|_2^2.
\end{align}
As $w_{k+1}$ is the optimal solution to \eqref{eq:w_subprob}, it is straightforward to have 
\begin{gather}
\begin{aligned}
    &\mathcal{L}_{\beta_k}(u_{k+1}, v_{k+1}, w_{k+1}, y_k, z_k) - \mathcal{L}_{\beta_k}(u_{k+1}, v_{k+1}, w_{k}, y_k, z_k)  \leq 0.
\end{aligned}
\end{gather}
Simple calculations by using  \eqref{eq:y_eq}-\eqref{eq:z_eq} lead to
\begin{gather}
\begin{aligned}
    \mathcal{L}_{\beta_k}(u_{k+1}, &v_{k+1}, w_{k+1}, y_{k+1}, z_{k+1}) - \mathcal{L}_{\beta_k}(u_{k+1}, v_{k+1}, w_{k+1}, y_k, z_k)  \\
     =&\left(\mathcal{L}_{\beta_k}(u_{k+1}, v_{k+1}, w_{k+1}, y_{k+1}, z_{k+1}) - \mathcal{L}_{\beta_k}(u_{k+1}, v_{k+1}, w_{k+1}, y_{k+1}, z_{k}) \right) \\
     &+\left(\mathcal{L}_{\beta_k}(u_{k+1}, v_{k+1}, w_{k+1}, y_{k+1}, z_{k}) - \mathcal{L}_{\beta_k}(u_{k+1}, v_{k+1}, w_{k+1}, y_{k}, z_{k}) \right) \\
     = &\langle z_{k+1} - z_k, \nabla u_{k+1} - w_{k+1} \rangle + \langle y_{k+1} - y_k, Au_{k+1} - v_{k+1} \rangle\\
     = & \frac{1}{\beta_k} \left(\|y_{k+1}-y_k\|_2^2 +\|z_{k+1} - z_k\|_2^2 \right).
\end{aligned}
\end{gather}
Lastly, we  have
\begin{gather}
\begin{aligned}\label{eq:ineq4}
    \mathcal{L}_{\beta_{k+1}}(u_{k+1}, &v_{k+1}, w_{k+1}, y_{k+1}, z_{k+1}) -\mathcal{L}_{\beta_k}(u_{k+1}, v_{k+1}, w_{k+1}, y_{k+1}, z_{k+1})\\ =& \frac{\beta_{k+1}- \beta_k}{2} \left( \|Au_{k+1}-v_{k+1}\|_2^2 + \|\nabla u_{k+1} - w_{k+1}\|_2^2 \right) \\
    =& \frac{\beta_{k+1}- \beta_k}{2\beta_k^2} \left(\left\|y_{k+1}-y_k \right\|_2^2 +\left\|z_{k+1}-z_k \right\|_2^2 \right).
\end{aligned}
\end{gather}
Combining \eqref{eq:ineq3} - \eqref{eq:ineq4} together with the fact that $\beta_k = \sigma^k \beta_0$ for $\sigma > 1$, we obtain
\begin{align*}
\mathcal{L}_{\beta_{k+1}}(u_{k+1}, &v_{k+1}, w_{k+1}, y_{k+1}, z_{k+1}) - \mathcal{L}_{\beta_k}(u_k, v_k, w_k, y_k, z_k) \\  
\leq &- \frac{\nu}{2} \|u_{k+1} - u_k\|_2^2- \frac{\beta_0}{2} \|v_{k+1} - v_k\|_2^2 + \frac{\beta_{k+1} + \beta_k}{2 \beta_k^2} \left(\left\|y_{k+1}-y_k \right\|_2^2 +\left\|z_{k+1}-z_k \right\|_2^2 \right)\\
=&-\frac{\nu}{2} \|u_{k+1} - u_k\|_2^2 - \frac{\beta_0}{2} \|v_{k+1} - v_k\|_2^2+ \frac{\sigma+1}{2 \sigma^k \beta_0}\left(\left\|y_{k+1}-y_k \right\|_2^2 +\left\|z_{k+1}-z_k \right\|_2^2 \right) \\
\leq& - \frac{\nu}{2} \|u_{k+1} - u_k\|_2^2 - \frac{\beta_0}{2} \|v_{k+1} - v_k\|_2^2+ \frac{1}{\sigma^{k-1} \beta_0}\left(\left\|y_{k+1}-y_k \right\|_2^2 +\left\|z_{k+1}-z_k \right\|_2^2 \right).
\end{align*}
This completes the proof.
\end{proof}
\begin{lemma}\label{lemma:result}
Suppose that $\text{ker}(A) \cap \text{ker}(\nabla) = \{0\}$ and $0 \leq \alpha < 1$. Let $\{(u_k, v_k, w_k, y_k, z_k)\}_{k=1}^{\infty}$ be generated by Algorithm \ref{alg:admm}. If $\{y_k\}_{k=1}^{\infty}$ bounded, then the sequence  $\{(u_k, v_k, w_k, y_k, z_k)\}_{k=1}^{\infty}$ is bounded,
 $u_{k+1} - u_k \rightarrow 0,$ and $v_{k+1} -v_k \rightarrow 0$. 
\end{lemma}
\begin{proof}
First we show that $\{z_k\}_{k=1}^{\infty}$ is bounded. Combining \eqref{eq:z_eq} with the first-order optimality condition of \eqref{eq:pixel_w_subprob}, we have
\begin{gather}
\begin{aligned} \label{eq:w_first_order_opt}
    (z_{k+1})_{i,j} =(z_k)_{i,j} + \beta_k\left( (\nabla u_{k+1})_{i,j} - (w_{k+1})_{i,j}\right) &\in \partial\left( \|(w_{k+1})_{i,j}\|_1-\alpha\|(w_{k+1})_{i,j}\|_2\right)\\
    &\subseteq \partial\left( \|(w_{k+1})_{i,j}\|_1\right)-\alpha\partial \left(\|(w_{k+1})_{i,j}\|_2\right),
\end{aligned}
\end{gather}
which implies that there exist $\xi_1 \in \partial\|(w_{k+1})_{i,j}\|_1$ and $\xi_2 \in \partial \|(w_{k+1})_{i,j}\|_2$ such that $(z_{k+1})_{i,j} = \xi_1 -\alpha \xi_2$ for each $(i,j) \in \Omega$. 
Recall that for $x \in \mathbb{R}^2$ the subgradients of the two norms are 
\begin{align} \label{eq:l1_subgrad}
   &     \partial\|x\|_1 = \left\{
    \xi\in \mathbb{R}^2 : \xi_i = \begin{cases}
    \text{sign}(x_i) &\text{ if } x_i \neq 0 \\
    \xi_i \in [-1,1] &\text{ if } x_i = 0
    \end{cases}\; \text{ for } i =1,2
    \right\},\\
 \label{eq:l2_subgrad}
   & \partial \|x\|_{2} = \left\{
    \xi \in \mathbb{R}^2: \xi = \begin{cases}
    \frac{x}{\|x\|_2} &\text{ if } x \neq 0 \\
    \in \{ \xi \in \mathbb{R}^2: \|\xi\|_2 \leq 1\} &\text{ if } x = 0
    \end{cases} \right\}.
\end{align}
Therefore, we have $\|\xi_1\|_{\infty} \leq 1, \|\xi_2\|_{\infty} \leq 1$, and hence $\|(z_{k+1})_{i,j}\|_{\infty} \leq 1+ \alpha$
(by the triangle inequality), i.e., $\{z_k\}_{k=1}^{\infty}$ is bounded.

By the assumption $\{(y_k)\}_{k=1}^{\infty}$ is bounded. There exist two constants $C_1, C_2 > 0$ such that $\|y_{k+1} - y_k\|_2^2 \leq C_1$, $\|z_{k+1}-  z_k\|_2^2 \leq C_1$, $\|y_k\|_2^2 \leq C_2$, and $\|z_k\|_2^2 \leq C_2$ for all $k \in \mathbb{N}$. Hence,  we have from \eqref{eq:lagrange_ineq} that
\begin{gather}
\begin{aligned}\label{ineq:descent}
\mathcal{L}_{\beta_{k+1}}(u_{k+1}, v_{k+1}, w_{k+1}, y_{k+1}, z_{k+1}) \leq& \mathcal{L}_{\beta_k}(u_k, v_k, w_k, y_k, z_k) - \frac{\nu}{2} \|u_{k+1} - u_k\|_2^2\\ &- \frac{\beta_0}{2} \|v_{k+1} - v_k\|_2^2+ \frac{2C_1}{\sigma^{k-1} \beta_0}.
\end{aligned}
\end{gather}
A telescoping summation of \eqref{ineq:descent}  leads to 
\begin{gather}
\begin{aligned}\label{eq:lagrange_upper}
    \mathcal{L}_{\beta_{k+1}}(u_{k+1}, v_{k+1}, w_{k+1}, y_{k+1}, z_{k+1}) \leq & \mathcal{L}_{\beta_{0}}(u_{0}, v_{0}, w_{0}, y_{0}, z_{0})+ \frac{2C_1}{\beta_0}\sum_{i=0}^{k}\frac{1}{\sigma^{i-1}}\\  &- \frac{\nu}{2}\sum_{i=0}^{k} \|u_{i+1} - u_i\|_2^2-\frac{\beta_0}{2}\sum_{i=0}^k  \|v_{i+1} - v_i\|_2^2.
\end{aligned}
\end{gather}
By completing two least-squares terms, we can rewrite $\mathcal{L}_{\beta_{k+1}}$ as 
\begin{gather}
\begin{aligned}
\label{eq:lagrange_rewrite}
    \mathcal{L}_{\beta_{k+1}}(u_{k+1}, v_{k+1}, w_{k+1}, y_{k+1}, z_{k+1}) 
    =  &\lambda \langle v_{k+1} - f \log v_{k+1}, \mathbbm{1} \rangle + \frac{\mu}{2} \|\nabla u_{k+1}\|_2^2 + \|w_{k+1} \|_1 - \alpha  \|w_{k+1}\|_{2,1}\\ &+ \frac{\beta_{k+1}}{2} \left \|Au_{k+1} - v_{k+1} + \frac{y_{k+1}}{\beta_{k+1}}\right\|_2^2 - \frac{\|y_{k+1}\|_2^2}{2\beta_{k+1}}\\
    &+ \frac{\beta_{k+1}}{2} \left\|\nabla u_{k+1} - w_{k+1} + \frac{z_{k+1}}{\beta_{k+1}} \right\|_2^2 - \frac{\|z_{k+1}\|_2^2}{2 \beta_{k+1}}.
    \end{aligned}
    \end{gather}
Combining \eqref{eq:lagrange_upper} and \eqref{eq:lagrange_rewrite}, we have
\begin{gather}
\begin{aligned} \label{eq:complex_ineq}
    \lambda \langle f - f \log f, \mathbbm{1} \rangle + (1-\alpha) \|w_{k+1}\|_1 - \frac{C_2}{\beta_{0}}  \leq &\mathcal{L}_{\beta_{k+1}}(u_{k+1}, v_{k+1}, w_{k+1}, y_{k+1}, z_{k+1})\\
    \leq &\mathcal{L}_{\beta_{0}}(u_{0}, v_{0}, w_{0}, y_{0}, z_{0})+ \frac{2C_1}{\beta_0}\sum_{i=0}^{k}\frac{1}{\sigma^{i-1}}\\  &- \frac{\nu}{2}\sum_{i=0}^{k} \|u_{i+1} - u_i\|_2^2-\frac{\beta_0}{2}\sum_{i=0}^k  \|v_{i+1} - v_i\|_2^2\\
    \leq & \mathcal{L}_{\beta_{0}}(u_{0}, v_{0}, w_{0}, y_{0}, z_{0})+ \frac{2C_1}{\beta_0}\sum_{i=0}^{\infty}\frac{1}{\sigma^{i-1}}.
\end{aligned}
\end{gather}
Since $\sigma > 1$,  the infinite sum is finite, and hence we have $\forall k \in \mathbb{N},$
\begin{align*}
    \|w_{k+1}\|_1 \leq \frac{1}{1-\alpha}\left(  \mathcal{L}_{\beta_{0}}(u_{0}, v_{0}, w_{0}, y_{0}, z_{0})-\lambda \langle f - f \log f, \mathbbm{1} \rangle+ \frac{2C_1}{\beta_0}\sum_{i=0}^{\infty}\frac{1}{\sigma^{i-1}} + \frac{C_2}{\beta_0} \right) < \infty,
\end{align*}
which implies that $\{w_k\}_{k=1}^{\infty}$ is bounded. Also from \eqref{eq:lagrange_upper} and \eqref{eq:lagrange_rewrite}, we have
\begin{align*}
    \lambda \langle f - f \log f, \mathbbm{1} \rangle - \frac{C_2}{\beta_0}&\leq \lambda \langle v_{k+1} - f \log v_{k+1}, \mathbbm{1} \rangle - \frac{C_2}{\beta_0}\\ &\leq \lambda \langle v_{k+1} - f \log v_{k+1}, \mathbbm{1} \rangle - \frac{\|y_{k+1}\|_2^2}{2 \beta_{k+1}}- \frac{\|z_{k+1}\|_2^2}{2 \beta_{k+1}}\\
    &\leq \mathcal{L}_{\beta_{k+1}}(u_{k+1}, v_{k+1}, w_{k+1}, y_{k+1}, z_{k+1}) \\
    &\leq \mathcal{L}_{\beta_{0}}(u_{0}, v_{0}, w_{0}, y_{0}, z_{0})+ \frac{2C_1}{\beta_0}\sum_{i=0}^{\infty}\frac{1}{\sigma^{i-1}} < \infty.
\end{align*}
This shows that $\{\langle v_k - f\log v_k, \mathbbm{1} \rangle\}_{k=1}^{\infty}$ is bounded. By emulating the computation in \eqref{eq:au_bound}, it can be shown that $\{v_k\}_{k=1}^{\infty}$ is bounded. 

It suffices to prove that $\{(Au_k, \nabla u_k )\}_{k=1}^{\infty}$ is bounded in order to  prove the boundedness of $\{u_k\}_{k=1}^{\infty}$ by Proposition \ref{prop:bound}. Using \eqref{eq:y_eq}, we have
\begin{align*}
    \|Au_{k+1}\|_2 \leq \frac{\|y_{k+1} -y_k\|_2}{\beta_k} + \|v_{k+1}\|_2 \leq \frac{\sqrt{C_1}}{\beta_0}+ \|v_{k+1}\|_2.
\end{align*}
As $\{v_{k}\}_{k=1}^{\infty}$ is proven to be bounded, then $\{Au_{k}\}_{k=1}^{\infty}$ is also bounded. We can prove $\{\nabla u_k\}_{k=1}^{\infty}$ is bounded similarly using \eqref{eq:z_eq}. Altogether,  $\{(u_k, v_k, w_k, y_k, z_k)\}_{k=1}^{\infty}$ is bounded.

It follows from \eqref{eq:complex_ineq} that
\begin{align*}
    \frac{\nu}{2} \sum_{i=0}^k \|u_{i+1} -u_i\|_2^2 +\frac{\beta_0}{2}\sum_{i=0}^k  \|v_{i+1} - v_i\|_2^2
   \leq  &\mathcal{L}_{\beta_{0}}(u_{0}, v_{0}, w_{0}, y_{0}, z_{0})\\&+ \frac{2C_1}{\beta_0}\sum_{i=0}^{k}\frac{1}{\sigma^{i-1}}-\lambda \langle f - f \log f, \mathbbm{1} \rangle+\frac{C_2}{\beta_0}. 
\end{align*}
As $k \rightarrow \infty$, we 
see the right-hand side is finite, which forces the infinite summations on the left-hand side to converge, and hence we have $u_{k+1} - u_k \rightarrow 0$ and $v_{k+1} -v_k \rightarrow 0$. 
\end{proof}

\begin{mythm}\label{thm:convergence_result}
Suppose that $\text{ker}(A) \cap \text{ker}(\nabla) = \{0\}$ and $0 \leq \alpha < 1$. Let $\{(u_k, v_k, w_k, y_k, z_k)\}_{k=1}^{\infty}$ be generated by Algorithm \ref{alg:admm}. If $\{y_k\}_{k=1}^{\infty}$ bounded,  $\beta_k(v_{k+1}-v_k) \rightarrow 0, \beta_k (w_{k+1} - w_k) \rightarrow 0, y_{k+1} - y_k \rightarrow 0$, and $z_{k+1} - z_k \rightarrow 0$, then there exists a subsequence  whose limit point $(u^*, v^*, w^*, y^*, z^*)$ is a  stationary point of \eqref{eq:constrained_opt} that satisfies the following:
\begin{subequations}
\begin{align}
    0 &=-\mu \Delta u^* + A^{\top}y^* + \nabla^{\top} z^*, \\
    0&=\lambda \left( \mathbbm{1}- \frac{f}{v^*} \right)- y^*,\\
     z^* &\in \partial \left( \|w^{*}\|_1 - \alpha \|w^{*}\|_{2,1} \right),\\
     Au^* &= v^*, \\
     \nabla u^* &= w^*.
\end{align}
\end{subequations}
\end{mythm}

\begin{proof}
By Lemma \ref{lemma:result}, the sequence $\{(u_k, v_k, w_k, y_k, z_k)\}_{k=1}^{\infty}$ is bounded, so there exists a subsequence $\{(u_{k_n}, v_{k_n}, w_{k_n}, y_{k_n}, z_{k_n})\}_{n=1}^{\infty}$ that converges to a point $(u^*, v^*, w^*, y^*, z^*)$. Additionally, we have $u_{k+1} - u_k \rightarrow 0$ and $v_{k+1}-v_k \rightarrow 0$.  Since $\{(y_{k}, z_{k})\}_{k=1}^{\infty}$ is bounded, there exists a constant $C > 0$ such that $\|y_{k+1} -y_k\|_2 < C$ and $\|z_{k+1} -z_k\|_2 < C$ for each $k \in \mathbb{N}$. By \eqref{eq:z_eq}, we have
\begin{align*}
    \|w_{k+1} - w_k \|_2 &\leq \|w_{k+1} - \nabla u_{k+1}\|_2 + \|\nabla u_{k+1} - \nabla u_k \|_2 + \|\nabla u_k - w_k\|_2\\ &= \frac{\|z_{k+1} - z_k\|_2}{\beta_k} + \|\nabla u_{k+1} - \nabla u_k \|_2 + \frac{\|z_{k} - z_{k-1}\|_2}{\beta_{k-1}}\\ &\leq \frac{2C}{\beta_{k-1}} + \|\nabla u_{k+1} - \nabla u_k \|_2.
\end{align*}
As $k\rightarrow \infty$, we have $w_{k+1} - w_k \rightarrow 0$. Altogether, we can derive the following results:
\begin{align}
    \lim_{n \rightarrow \infty} (u_{k_n+1}, v_{k_n+1}, w_{k_n+1}) = \lim_{n \rightarrow \infty} (u_{k_n}, v_{k_n}, w_{k_n}) = (u^*, v^*, w^*).
\end{align}
Furthermore, the assumptions give us
\begin{align*}
    \lim_{n \rightarrow \infty} \beta_{k_n} (v_{k_n+1}-v_{k_n}) &= 0,\\
    \lim_{n \rightarrow \infty} \beta_{k_n} (w_{k_n+1}-w_{k_n}) &= 0,\\
    \lim_{n \rightarrow \infty} y_{k_n+1}-y_{k_n} &=0,\\
    \lim_{n \rightarrow \infty} z_{k_n+1} - z_{k_n} &= 0.
\end{align*}
 By \eqref{eq:y_eq}-\eqref{eq:z_eq}, we have
\begin{align*}
     \|Au^* - v^*\|_2 &= \lim_{n \rightarrow \infty} \|Au_{k_n+1} - v_{k_n+1}\|_2 = \lim_{n \rightarrow \infty}\frac{\|y_{k_n+1} - y_{k_n}\|_2}{\beta_{k_n}} \leq \lim_{n \rightarrow \infty}\frac{C}{\beta_{k_n}} =0, \\
     \|\nabla u^* - w^* \|_2 &= \lim_{n \rightarrow \infty} \|\nabla u_{k_n+1} - w_{k_n+1}\|_2 =\lim_{n \rightarrow \infty} \frac{\|z_{k+1} - z_k\|_2}{\beta_{k_n}} \leq\lim_{n \rightarrow \infty} \frac{C}{\beta_{k_n}} = 0.
\end{align*}
Hence, we have $Au^* = v^*$ and $\nabla u^* = w^*$. 

The optimality conditions at iteration $k_n$ are the following:
\begin{subequations}\label{eq:limit_argument}
\begin{align}
    &-\mu \Delta u_{k_n+1}+A^{\top}y_{k_n} + \beta_{k_n} A^{\top} (Au_{k_n+1} -v_{k_n}) + \nabla^{\top}z_{k_n} + \beta_{k_n} \nabla^{\top}(\nabla u_{k_n+1}-w_{k_n})=0,\label{eq:k_n_first} \\
    &\lambda \left( \mathbbm{1} - \frac{f}{v_{k_n+1}} \right) - y_{k_n} - \beta_{k_n} (Au_{k_n+1}-v_{k_n+1})=0, \label{eq:k_n_second}\\
    &z_{k_n} + \beta_{k_n}(\nabla u_{k_n+1}- w_{k_n+1}) \in \partial(\|w_{k_n+1}\|_1 - \alpha \|w_{k_n+1}\|_{2,1}). \label{eq:k_n_third}
    \end{align}
\end{subequations}
Expanding \eqref{eq:k_n_first} by substituting in \eqref{eq:y_eq}-\eqref{eq:z_eq} and taking the limit, we have
\begin{align*}
    0 =& \lim_{n \rightarrow \infty} -\mu \Delta u_{k_n+1}+A^{\top}y_{k_n} + \beta_{k_n} A^{\top} (Au_{k_n+1} -v_{k_n}) + \nabla^{\top}z_{k_n} + \beta_{k_n} \nabla^{\top}(\nabla u_{k_n+1}-w_{k_n})\\
    =&\lim_{n \rightarrow \infty} -\mu \Delta u_{k_n+1}+A^{\top}y_{k_n} + \beta_{k_n} A^{\top} (Au_{k_n+1} -v_{k_n+1}) + \beta_{k_n} A^{\top} (v_{k_n+1} -v_{k_n}) + \nabla^{\top}z_{k_n}\\ &+ \beta_{k_n} \nabla^{\top}(\nabla u_{k_n+1}-w_{k_n+1})+\beta_{k_n} \nabla^{\top}( w_{k_n+1}-w_{k_n})\\
    =&\lim_{n \rightarrow \infty} -\mu \Delta u_{k_n+1}+A^{\top}y_{k_n} + A^{\top} (y_{k_n+1} -y_{k_n}) + \beta_{k_n} A^{\top} (v_{k_n+1} -v_{k_n}) + \nabla^{\top}z_{k_n}\\ &+  \nabla^{\top}( z_{k_n+1}-z_{k_n})+\beta_{k_n} \nabla^{\top}( w_{k_n+1}-w_{k_n})\\
    =& -\mu\Delta u^* + A^{\top} y^* + \nabla^{\top} z^*.
\end{align*}
Substituting in \eqref{eq:y_eq} into \eqref{eq:k_n_second} and taking the limit give us
\begin{align*}
    0 =& \lim_{n \rightarrow \infty} \lambda \left( \mathbbm{1} - \frac{f}{v_{k_n+1}} \right) - y_{k_n} - \beta_{k_n} (Au_{k_n+1}-v_{k_n+1})\\ =& \lim_{n \rightarrow \infty} \lambda \left( \mathbbm{1} - \frac{f}{v_{k_n+1}} \right) - y_{k_n} - (y_{k_n+1} -y_{k_n})\\ =& \lambda \left( \mathbbm{1}- \frac{f}{v^*} \right)- y^*.
\end{align*}
Lastly, by substituting \eqref{eq:z_eq} into \eqref{eq:k_n_third}, we have
\begin{align*}
    z_{k_n+1} \in \partial(\|w_{k_n+1}\|_1 - \alpha \|w_{k_n+1}\|_{2,1}).
\end{align*}
By continuity, we have $\|w_{k_n+1}\|_1 - \alpha \|w_{k_n+1}\|_{2,1} \rightarrow \|w^*\|_1 - \alpha \|w^*\|_{2,1}$. Together with the fact that $(w_{k_n+1} , z_{k_n+1}) \rightarrow (w^*, z^*)$, we have $ z^* \in \partial \left( \|w^{*}\|_1 - \alpha \|w^{*}\|_{2,1} \right)$ by closedness of the subdifferential. 

Therefore, $(u^*, v^*, w^*, y^*, z^*)$ is a stationary point.
\end{proof}
\begin{remark}
It is true that the assumptions in Theorem \ref{thm:convergence_result} are rather strong, but they are standard in 
 the convergence analyses of other ADMM algorithms for nonconvex problems that fail to satisfy the conditions for global convergence in \cite{wang2019global}.
For example, \cite{jung2017piecewise, jung2014variational,  li2016multiphase, li2020tv} assumed convergence of the successive differences of the primal variables and Lagrange multipliers. Instead, we modify the convergence of the successive difference of the primal variables, i.e., $\beta_k(v_{k+1}-v_k) \rightarrow 0, \beta_k (w_{k+1} - w_k) \rightarrow 0$.
Boundedness of the Lagrange multiplier (i.e., $\{y_k\}_{k=1}^{\infty}$) was also assumed in \cite{liu2022entropy, xu2012alternating}, which required a stronger assumption than ours regarding the successive difference of the Lagrange multipliers. 
\end{remark}
\section{Numerical Experiments} \label{sec:experiment}

In this section, we apply the proposed method of AITV Poisson SaT/SLaT on various grayscale and color images for image segmentation. For grayscale images, we compare our method with the original TV SaT \cite{chan2014two}, thresholded-Rudin-Osher-Fatemi (T-ROF) \cite{cai2019linkage}, and the Potts model \cite{potts1952some} solved by either Pock's algorithm (Pock) \cite{pock2009convex}  or Storath and Weinmann's algorithm (Storath) \cite{storath2014fast}. For color images, we compare   with TV SLaT \cite{cai2017three}, Pock's method \cite{pock2009convex}, and Storath's method \cite{storath2014fast}. We can solve \eqref{eq:ms_poisson_discrete} for TV SaT/SLaT via Algorithm \ref{alg:admm} that utilizes the proximal operator corresponding to the $\|\cdot\|_{2,1}$ norm. The code for T-ROF is provided by the respective author\footnote{\url{https://xiaohaocai.netlify.app/download/}} and we can   adapt it to handle blur by using a more general data fidelity term. Pock's method is implemented by the lab group\footnote{Python code is available at \url{https://github.com/VLOGroup/pgmo-lecture/blob/master/notebooks/tv-potts.ipynb} and a translated MATLAB code is available at \url{https://github.com/kbui1993/MATLAB_Potts}.}. Storath's method is provided by the original author\footnote{ \url{https://github.com/mstorath/Pottslab}}. Note that T-ROF, Pock's method, and Storath's method are designed for images corrupted with Gaussian noise. We apply the Anscombe transform \cite{anscombe1948transformation} to the test images, after which the Poisson noise becomes approximately Gaussian noise. Since Storath's method is not for segmentation, we perform a post-processing step of $k$-means clustering to its  piecewise-constant output. For the SLaT methods, we 
parallelize the smoothing step separately  for each channel. 

\begin{figure*}[t!]
    \centering
    \includegraphics[scale=1.3]{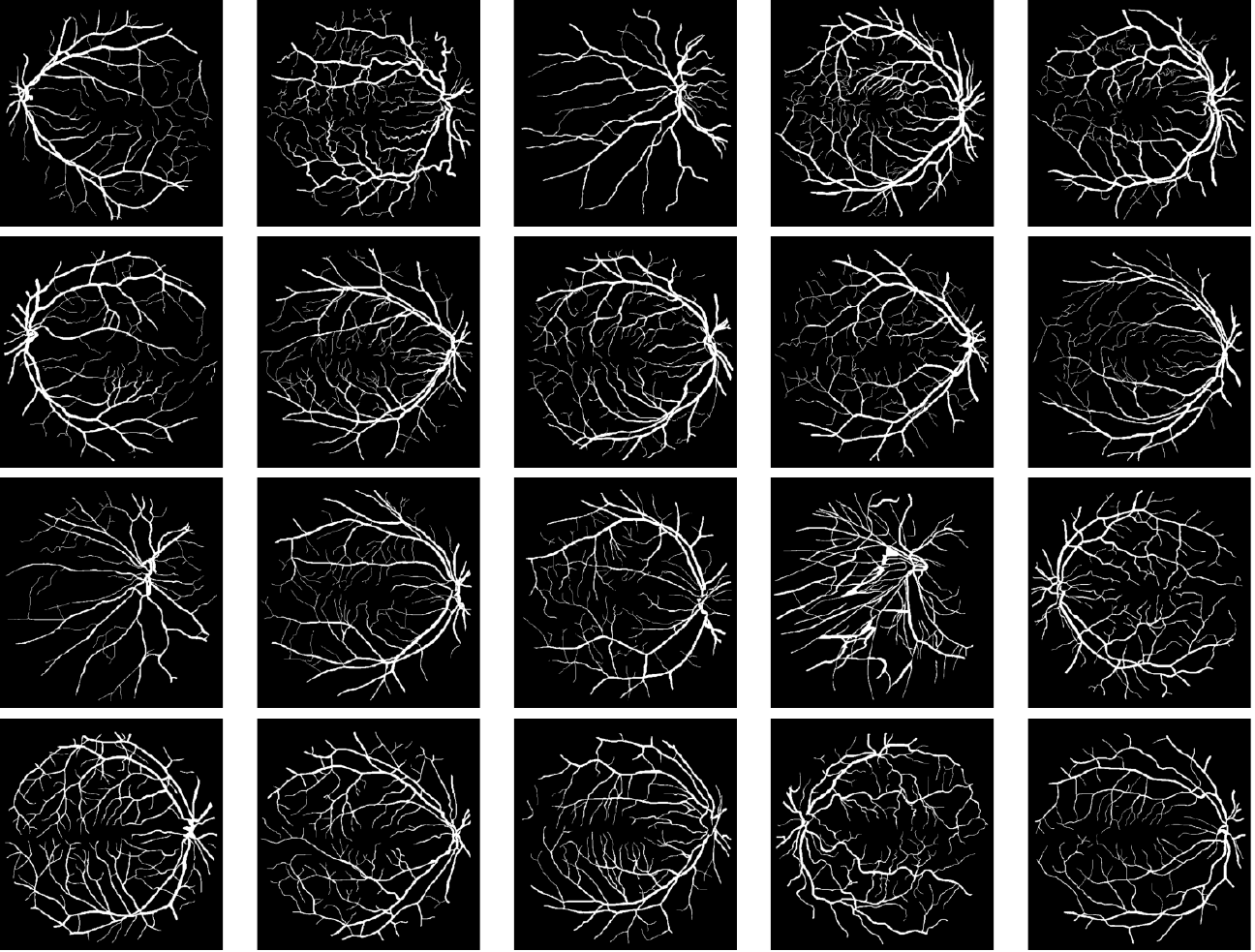}
    \caption{The entire DRIVE dataset \mbox{\cite{staal2004ridge}} for binary segmentation. The image size is $584 \times 565$ with background value of 200 and the pixel value for vessels to be 255.}
    \label{fig:retinex_test}
\end{figure*}

To quantitatively measure the segmentation performance, we use the DICE index \cite{dice1945measures} and peak signal-to-noise ratio (PSNR). Let $S \subset \Omega$ be the ground-truth region and $S' \subset \Omega$ be a region obtained from the segmentation algorithm corresponding to the ground-truth region $S$. The DICE index is formulated by
\begin{align*}
\text{DICE} = \frac{2 |S \cap S'|}{|S|+|S'|}.
\end{align*}
To compare the piecewise-constant reconstruction $\tilde{f}$ according to \eqref{eq:pc_constant_approx} with the original test image $f$, we compute PSNR by 
\begin{align*}
\text{PSNR} = 20 \log_{10} \frac{(M \times N) \times P}{\sum_{i,j} (f_{i,j} - \tilde{f}_{i,j})^2}, 
\end{align*}
where $M \times N$ is the image size and $P = \max_{i,j} f_{i,j}$.

Poisson noise is added to the test images by the MATLAB command \texttt{poissrnd}.
To ease parameter tuning, we scale each test image to $[0,1]$ after its degradation with Poisson noise and/or blur. 
We set $\sigma = 1.25$ and $\beta_{1,0}=\beta_{2,0} = 1.0,2.0$ in Algorithm \ref{alg:admm} for grayscale and color images, respectively. The stopping criterion is either 300 iterations or when the relative error of $u_k$ is below $\epsilon = 10^{-4}$. We tune
the fidelity parameter $\lambda$ and the smoothing parameter $\mu$
for each image,  which will be specified later. For T-ROF, Pock's method, and Storath's method, their parameters are manually tuned to give the best DICE indices for binary segmentation (Section \ref{sec:binary}) and the PSNR values for multiphase segmentation (Section \ref{sec:multi}-\ref{sec:color}). 
All experiments are performed in MATLAB R2022b on a Dell laptop with a 1.80 GHz Intel Core i7-8565U processor and 16.0 GB RAM.

\subsection{Grayscale, Binary Segmentation} \label{sec:binary}

We start with performing binary segmentation on the entire DRIVE dataset \mbox{\cite{staal2004ridge}} that consists of 20  images shown in Figure \mbox{\ref{fig:retinex_test}}. Each image has size $584 \times 565$ with modified pixel values of either 200 for the background or 255 for the vessels. Before adding Poisson noise, we set the peak value of the image to be $P/2$ or $P/5$, where $P=255$. Note that a lower peak value indicates  stronger noise in the image, thus more challenging for denoising. We examine three cases: (1) $P/2$ no blur, (2) $P/5$ no blur, and (3) $P/2$ with Gaussian blur specified by MATLAB command \texttt{fspecial('gaussian', [10 10], 2)}. For the TV SaT method, we set $\lambda = 14.5, \ \mu = 0.5$ for case (1), $\lambda = 8.0, \ \mu = 0.5$ for case (2), and $\lambda =22.5, \ \mu = 0.25$ for case (3). For the AITV SaT method, the parameters $\lambda$ and $\mu$ are set the same as TV SaT, and we have $\alpha = 0.3$ for cases (1)-(2) and $\alpha = 0.8$ for case (3). 

Table \mbox{\ref{tab:new_binary_dice}} records the DICE indices and the computational time in seconds for the competing methods, averaged over 20 images. We observe that AITV SaT attains the best DICE indices for all three cases with comparable computational time to TV SaT and T-ROF, all of which are much faster than Pock and Storath. As visually illustrated in Figure~\mbox{\ref{fig:p_2_blur}}, AITV SaT segments more of the thinner vessels compared to TV SaT and T-ROF in five images, thereby having the higher average DICE indices.

\begin{figure}[t]
    \centering
    \includegraphics[scale=1.3]{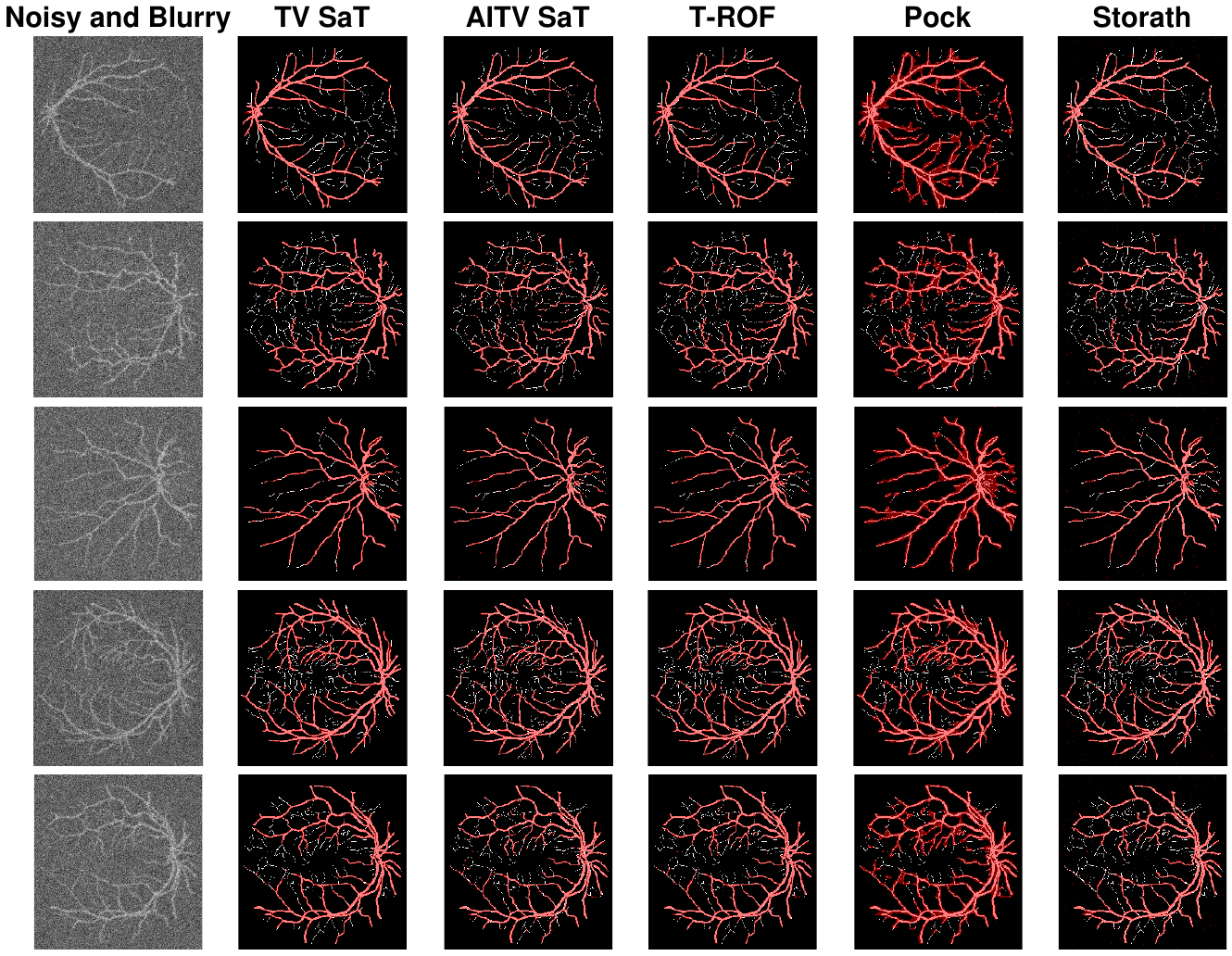}
    \caption{Binary segmentation results of Figure \mbox{\ref{fig:retinex_test}} with peak $P/2$ under Gaussian blur  and  Poisson noise.}
    \label{fig:p_2_blur}
\end{figure}
\begin{table*}[t!!]
    \centering
    \caption{DICE and computational time in seconds of the binary segmentation methods averaged over 20 images in Figure \mbox{\ref{fig:retinex_test}} with standard deviations in parentheses. \textbf{Bold} indicates best result.} \label{tab:new_binary_dice}
    \begin{tabular}{|l|l|c|c|c|c|c|}
    \hline
       & &TV SaT & AITV SaT & T-ROF & Pock & Storath \\ \hline
        \multirow{ 2}{*}{$P/2$ no blur}& DICE & \makecell{0.9464 \\(0.0091)} & \makecell{\textbf{0.9501} \\ (0.0076)} & \makecell{0.9463 \\(0.0073)} & \makecell{0.8466\\ (0.0301)} & \makecell{0.8855\\ (0.0181)}
 \\ \cline{2-7}
        & Time (sec.)  & \makecell{\textbf{4.2401} \\ (0.3618)} & \makecell{5.7342 \\ (0.5251)} & \makecell{4.9206 \\(1.4281)} & \makecell{24.7376	\\ (3.2454)} & \makecell{19.9456\\ (1.8875)}
 \\  \hline
        \multirow{ 2}{*}{$P/5$ no blur}& DICE & \makecell{0.8714 \\ (0.0134)} &	\makecell{\textbf{0.8735} \\ (0.0125)} & \makecell{0.8570 \\ (0.0170)}	& \makecell{0.6504 \\ (0.0910)} & \makecell{0.8277 \\ (0.0191)}
 \\ \cline{2-7}
        &Time (sec.) & \makecell{\textbf{4.7076} \\ (0.6937)}	& \makecell{6.4027 \\ (0.8441)} & \makecell{5.4943 \\ (0.7935)} & \makecell{46.9346 \\ (9.4969)}	& \makecell{21.8734  \\(2.7660)} \\ \hline
        \multirow{ 2}{*}{$P/2 $ with Gaussian Blur}& DICE & \makecell{0.7244 \\ (0.0254)}	& \makecell{\textbf{0.7411} \\ (0.0220)}	& \makecell{0.7322 \\ (0.0251)}	& \makecell{0.5473 \\ (0.0398)} &	\makecell{0.6944 \\ (0.0217)} \\ \cline{2-7}

       &Time (sec.) & \makecell{\textbf{7.4495} \\ (1.0983)}	& \makecell{9.2523 \\ (1.5959)} & \makecell{11.7337 \\ (2.1252)} & \makecell{47.3911 \\ (10.9191)}	& \makecell{19.9444 \\ (2.4142)}
 \\ \hline
    \end{tabular}
\end{table*}

\subsection{Grayscale, Multiphase Segmentation} \label{sec:multi}

\begin{figure}[t!!]
    \centering
    \includegraphics[scale=1.3]{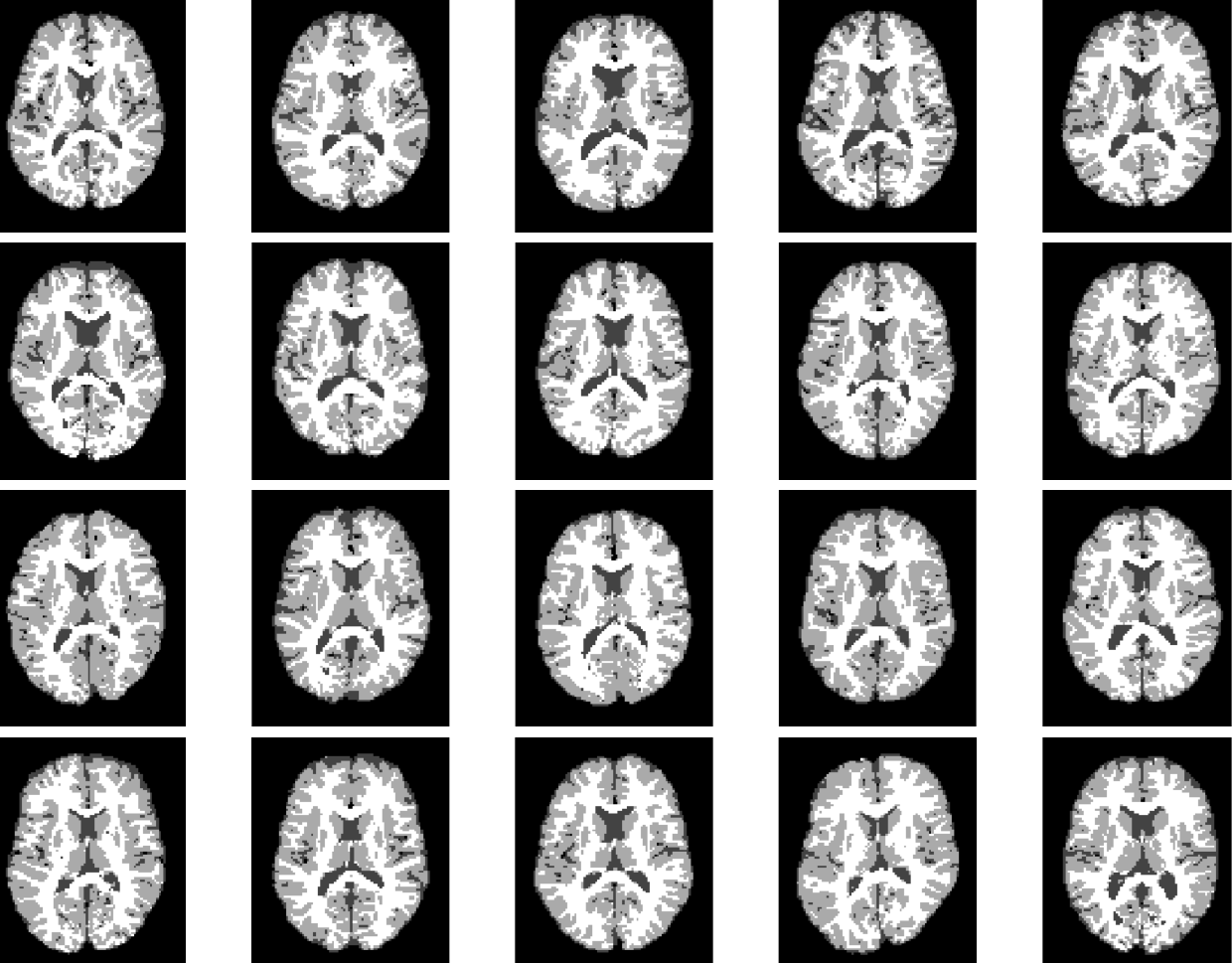}
    \caption{The entire BrainWeb dataset \mbox{\cite{aubert2006twenty}} for grayscale, multiphase segmentation. Each image is of size $104 \times 87$. The pixel values are 10 (background), 48 (cerebrospinal fluid), 106 (grey matter), and 154 (white matter).}
    \label{fig:brain}
\end{figure}

\begin{figure}[t]
    \centering
    \includegraphics[scale=1.2]{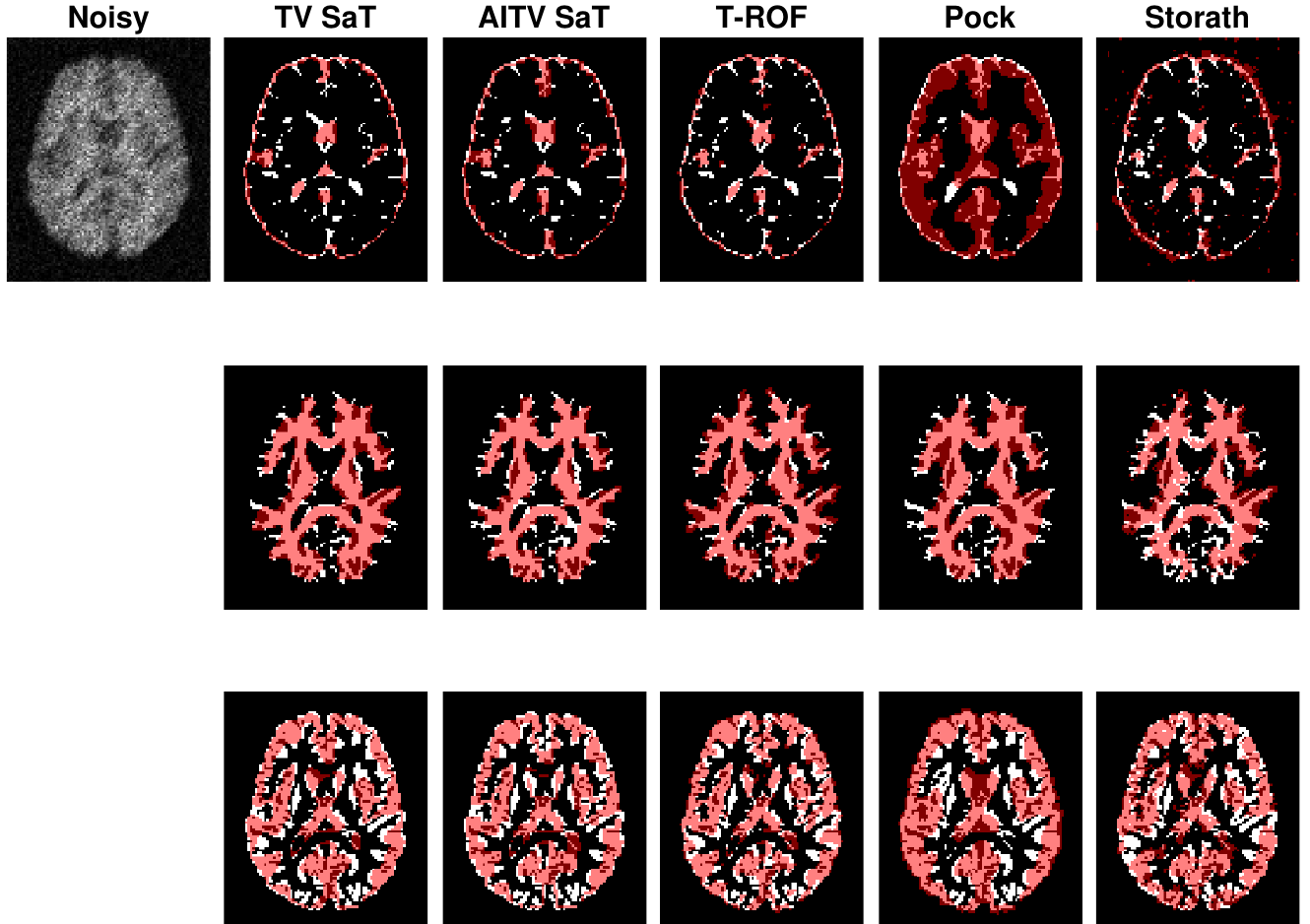}
    \caption{Segmentation result of the first image of Figure \mbox{\ref{fig:brain}} with peak $P/8$ under Poisson noise with no blur. From top to bottom are segmentation results for CSF, GM, and WM. }
    \label{fig:multiphase}
\end{figure}
\begin{table*}[t!!]
    \centering
    \caption{DICE and computational time in seconds of the multiphase segmentation methods averaged over 20 images in Figure \mbox{\ref{fig:brain}} with standard deviations in parentheses. \textbf{Bold} indicates best result.} \label{tab:new_brain_dice}
    \begin{tabular}{|l|l|c|c|c|c|c|}
    \hline
       & &TV SaT & AITV SaT & T-ROF & Pock & Storath \\ \hline
        \multirow{ 6}{*}{$P/8$ no blur}& CSF DICE & \makecell{0.8208\\(0.0270)}	& \makecell{\textbf{0.8396} \\ (0.0262)}	& \makecell{0.7398 \\ (0.1316)}	& \makecell{0.4572 \\ (0.1474)} &	\makecell{0.8041 \\ (0.0244)}

 \\ \cline{2-7}
 & GM DICE  & \makecell{0.8097 \\ (0.0258) }	& \makecell{\textbf{0.8477} \\ (0.0157)}	& \makecell{0.7904 \\ (0.0422)}	& \makecell{0.7507 \\ (0.0719)}	& \makecell{0.7900 \\ (0.0344)}

 \\ \cline{2-7}
  & WM DICE  & \makecell{0.8449 \\ (0.0125)} &	\makecell{\textbf{0.8694} \\ (0.0101)} &	\makecell{0.8221 \\ (0.0132)} &	\makecell{0.8459 \\ (0.0283)} &	\makecell{0.8138 \\ (0.0196)}

 \\ \cline{2-7}
        & Time (sec.)  & \makecell{0.2607 \\ (0.0625)} &	\makecell{0.2863 \\ (0.0470)} &	\makecell{\textbf{0.2202} \\ (0.0324)} & 	\makecell{2.6139 \\ (0.6690)} & 	\makecell{0.3383 \\ (0.0948)}

 \\  \hline
        \multirow{ 6}{*}{$P/8$ with Motion Blur}& CSF DICE & \makecell{0.6196 \\ (0.0385) } &	\makecell{\textbf{0.6260} \\ (0.0483)} &	\makecell{0.6174 \\ (0.0460)} &  \makecell{0.3772 \\ (0.0561) } &	\makecell{0.4964 \\ (0.0468)}

 \\ \cline{2-7}
         &GM DICE & \makecell{0.6809 \\ (0.0304) } &	\makecell{\textbf{0.7138} \\ (0.0262) } &	\makecell{0.6528 \\ (0.0358)} &	\makecell{0.6345 \\ (0.0590)}&	\makecell{0.6544\\ (0.0399)}
 \\\cline{2-7}
          &WM DICE & \makecell{0.7757 \\ (0.0127)}	&\makecell{\textbf{0.7935} \\ (0.0110)}	& \makecell{0.7686 \\ (0.0164)} & \makecell{0.7529 \\ (0.0185) }	& \makecell{0.7382 \\ (0.0140) }

 \\\cline{2-7}
        &Time (sec.) & \makecell{\textbf{0.2494} \\ (0.0443) } &	\makecell{0.2854 \\ (0.0279)}	& \makecell{0.3647 \\ (0.0693)} &	\makecell{2.5782\\ (0.4683)} & \makecell{0.3052 \\ (0.1399)}
\\ \hline
    \end{tabular}
\end{table*}

We examine  the multiphase segmentation on the entire BrainWeb dataset \mbox{\cite{aubert2006twenty}} that consists of 20 grayscale images as shown in Figure \mbox{\ref{fig:brain}}.
Each image is of size $104 \times 87$ and  has  four regions to segment: background, cerebrospinal fluid (CSF), grey matter (GM), and white matter (WM). The pixel values are 10 (background), 48 (CSF), 106 (GM), and 154 (WM). The maximum intensity $P = 154$.  We consider two cases: (1) $P/2$ no blur and (2) $P/2$ with motion blur specified by \texttt{fspecial('motion', 5, 225)}. For the SaT methods, we have $\mu = 1.0$, $\alpha = 0.6, 0.7$, and $\lambda = 4.0, 5.0$ for case (1) and case (2), respectively. 

Across all 20 images of the BrainWeb dataset, Table \mbox{\ref{tab:new_brain_dice}} reports the average DICE indices for CSF, GM, and WM and average computational times in seconds of the segmentation methods. For both cases (1) and (2), AITV SaT attains the highest average DICE indices for segmenting CSF, GM, and WM. AITV SaT is comparable to  TV SaT and T-ROF in terms of computational time.

Figure \mbox{\ref{fig:multiphase}} shows the segmentation results of the first image in Figure \mbox{\ref{fig:brain}} for case (1). When segmenting CSF, the methods (TV SaT, AITV SaT, and Storath) yield similar visual results, while Pock fails to segment roughly half of the region. In addition, AITV SaT segments the most GM region with the least amount of noise artifacts than the other methods. Lastly, for WM segmentation, AITV SaT avoids the ``holes" or ``gaps" and segments fewer regions 
outside of the ground truth, thus outperforming TV SaT and Storath. For the three regions, T-ROF has the most noise artifacts in its segmentation results.

\subsection{Color Segmentation}\label{sec:color}
We perform color image segmentation on 10 images shown in Figure \mbox{\ref{fig:color_test}}, which are selected from the PASCAL VOC 2010 dataset \mbox{\cite{everingham2009pascal}}. Each image has six different color regions. Figure 5(A) is of size $307 \times 461$; Figure 5(B) is of $500 \times 367$; Figures 5(C) and 5(H)-(I) are of $500 \times 375$; and Figures 5(D)-(G) and 5(J) are of $375 \times 500$. Before adding Poisson noise to each channel of each image, we set the peak value $P = 10$. We choose the parameters of the SLaT methods to be $\lambda = 1.5$, $\mu = 0.05$, and $\alpha = 0.6$. 

Figures \mbox{\ref{fig:color_result1}-\ref{fig:color_result2}} present the piecewise-constant approximations via \mbox{\eqref{eq:pc_constant_approx}},  showing similar segmentation results obtained by TV SLaT, AITV SLaT, and Storath. Quantitatively in Table \mbox{\ref{tab:color_result}}, AITV SLaT has better PSNRs than TV SLaT and Pock for all the images and outperforms Storath for seven. Overall, the proposed 
AITV SLaT has the highest PSNR on average over 10 images with the lowest standard deviation and 
comparable speed as Storath.

\begin{figure}[t!]
    \centering\includegraphics[scale=1.]{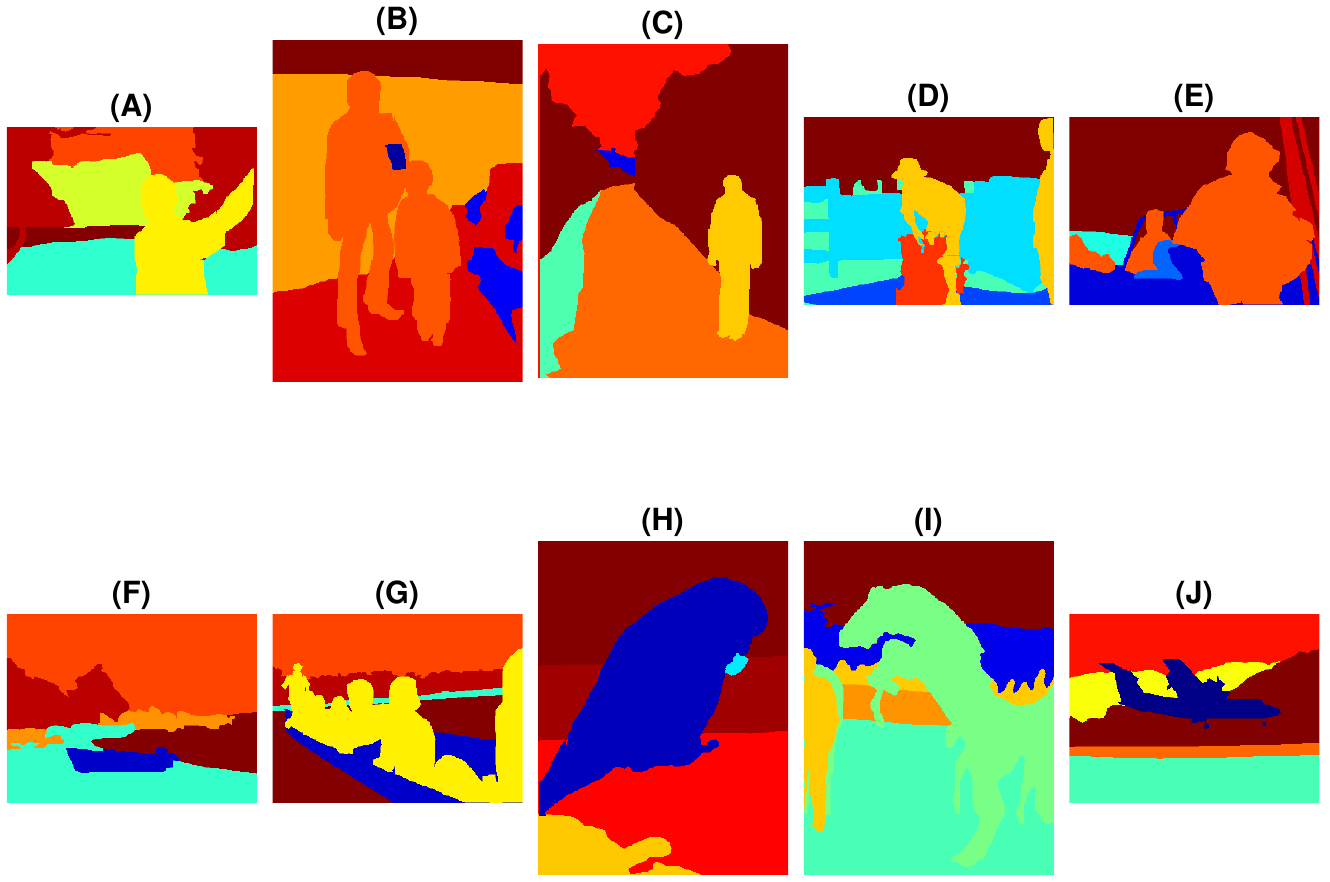}
    \caption{Test images from the PASCAL VOC 2010 dataset \mbox{\cite{everingham2009pascal}} for color, multiphase segmentation. Each image has 6 regions. The image sizes are (A) $307 \times 461$, (B) $500 \times 367$, (C) $500 \times 375$, (D)-(G) $375 \times 500$, (H)-(I) $500 \times 375$, and (J) $375 \times 500$.} 
    \label{fig:color_test}
\end{figure}

\begin{figure}[t!]
    \centering\includegraphics[scale=0.975]{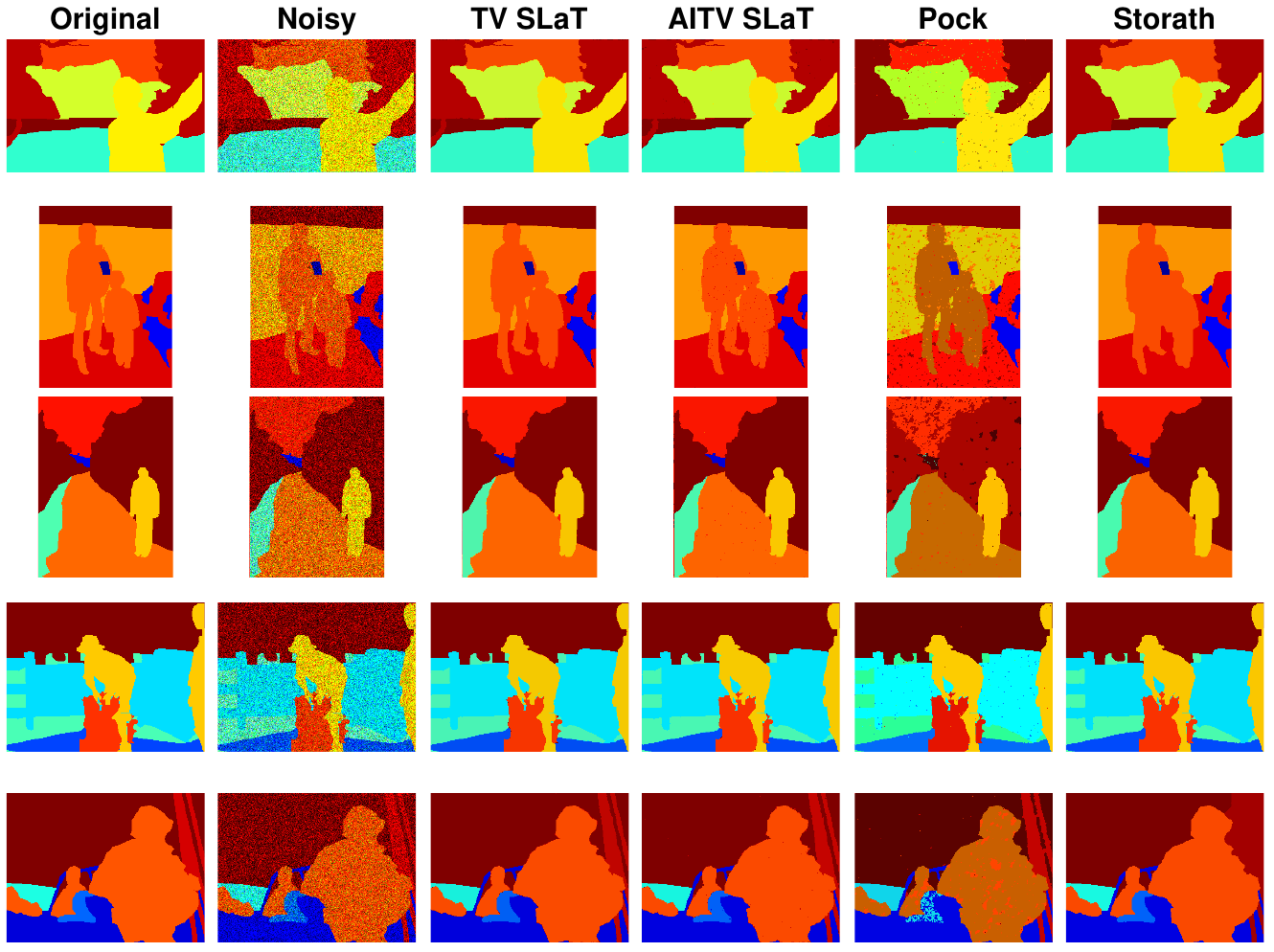}
    \caption{Color image segmentation results 
    of Figures \mbox{\ref{fig:color_test}}(A)-(E).} 
    \label{fig:color_result1}
\end{figure}
\begin{figure}[th!]
    \centering\includegraphics[scale=0.975]{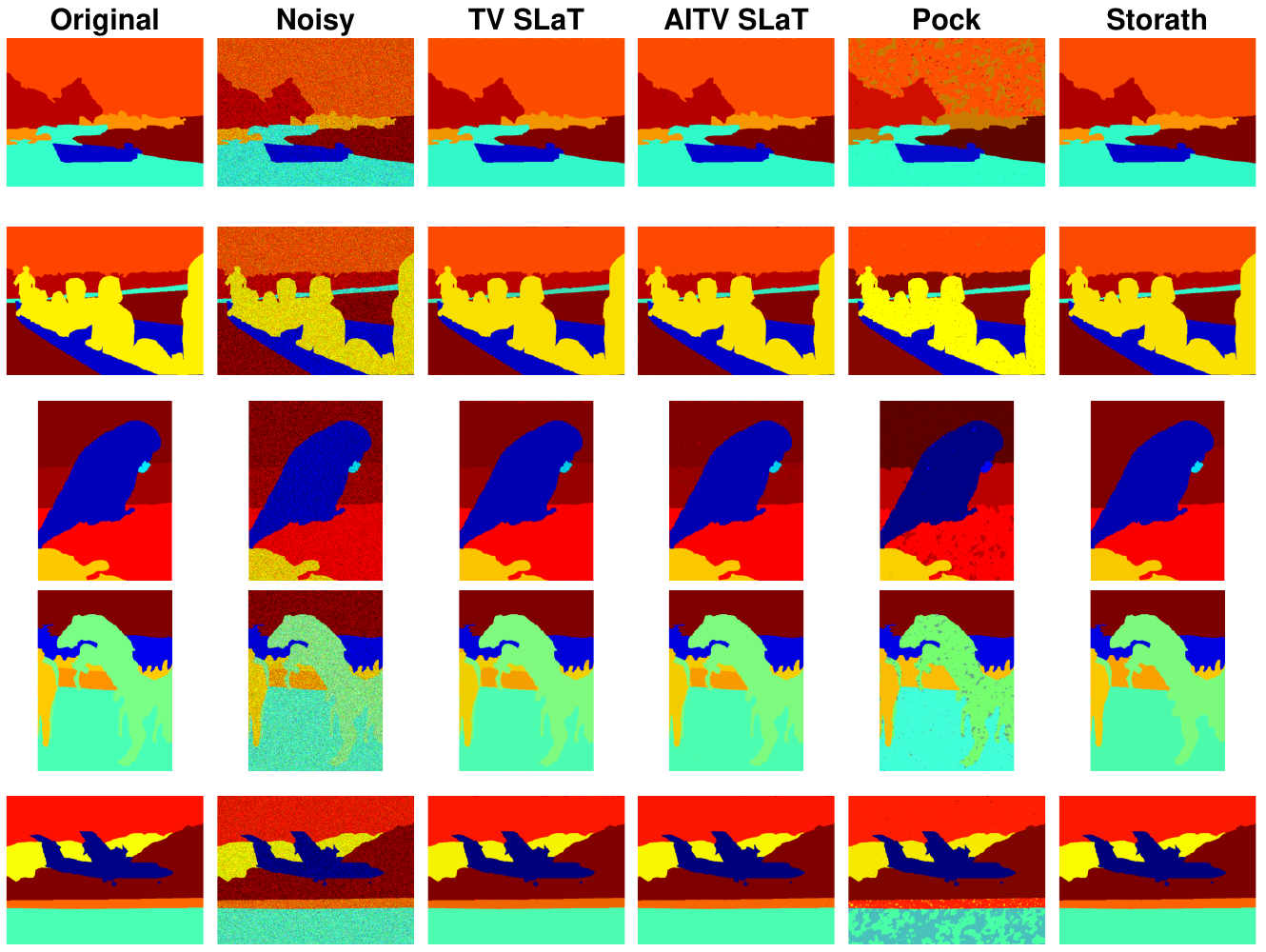}
    \caption{Color image segmentation results 
    of Figures \mbox{\ref{fig:color_test}}(F)-(J). } 
    \label{fig:color_result2}
\end{figure}

\begin{table}[t!!!]
 \caption{PSNR and computation time in seconds for color image segmentation of the images in Figure \mbox{\ref{fig:color_test}}. \textbf{Bold} indicates best result.}\label{tab:color_result}
    \centering
    \begin{tabular}{|l|l|c|c|c|c|}
    \hline
        &&TV SLaT & AITV SLaT & Pock & Storath \\ \hline
        \multirow{10}{*}{PSNR}&Figure \ref{fig:color_test}(A) & 29.0337	&\textbf{30.4619} &19.2306	& 28.7141
\\ \cline{2-6}
        &Figure \ref{fig:color_test}(B) & \textbf{31.8673}	& 31.7055	&17.3512	&30.9703
 \\ \cline{2-6}
        &Figure \ref{fig:color_test}(C) & 30.9605	& \textbf{33.4217}	&18.0389	& 32.2217
 \\ \cline{2-6}
        &Figure \ref{fig:color_test}(D) & 29.5265	&32.5505 &	21.4873&	\textbf{34.7881}\\ \cline{2-6}
        &Figure \ref{fig:color_test}(E) & 29.8903	&\textbf{31.0656}	&19.7646 &	28.3456\\ \cline{2-6}
        &Figure \ref{fig:color_test}(F) & 33.2308	& \textbf{34.7619}	& 17.9788	& 34.4106
\\ \cline{2-6}
        &Figure \ref{fig:color_test}(G) & 28.1136 &	30.6237 &	22.5048	& \textbf{31.2439}
\\ \cline{2-6}
        &Figure \ref{fig:color_test}(H) & 33.2682	& \textbf{33.4250}	& 18.9390	& 31.9377
\\ \cline{2-6}
        &Figure \ref{fig:color_test}(I) &30.0659	&\textbf{31.7937}&	20.0856&	29.3905
\\ \cline{2-6}
        &Figure \ref{fig:color_test}(J) & 31.4164	&34.0610&20.3185	&\textbf{34.3599}
        \\ \cline{2-6}
        &\makecell{Avg.\\(Std.)} & \makecell{30.7373 \\ (1.7266)}	&\makecell{\textbf{32.3870} \\ (1.4905)}	&\makecell{19.5700\\(1.6147)} &\makecell{31.6383\\(2.3686)}

 \\ \hline \hline
        \multirow{ 10}{*}{Time (sec.)}&Figure \ref{fig:color_test}(A) & 7.1992	&8.2455	&148.9232	& \textbf{4.9410}

\\ \cline{2-6}
        &Figure \ref{fig:color_test}(B) & 12.2238	&11.8890	&321.0637	&\textbf{9.9533}

 \\ \cline{2-6}
        &Figure \ref{fig:color_test}(C)&\textbf{5.6775}	&7.4494	&384.6998	&10.1785

 \\ \cline{2-6}
        & Figure \ref{fig:color_test}(D) & \textbf{4.9511} &	7.2061	&128.4236	&7.0026
\\ \cline{2-6}
         & Figure \ref{fig:color_test}(E)& \textbf{6.2380}	&7.0719	&275.0501 &	7.9808

\\ \cline{2-6} &Figure \ref{fig:color_test}(F) &\textbf{6.0509}	&7.2392&	398.1339&	10.0982

\\ \cline{2-6} & Figure \ref{fig:color_test}(G) & \textbf{7.2821}	&7.6919 &	165.2142 &	7.9315

\\ \cline{2-6}
 &Figure \ref{fig:color_test}(H) & \textbf{6.0269}	&9.7018&	385.3311	& 7.9389

\\ \cline{2-6} &Figure \ref{fig:color_test}(I) & \textbf{7.0375}	&8.0875 &	250.3698	&9.7501

\\ \cline{2-6} & Figure \ref{fig:color_test}(J) &\textbf{6.1873}	&10.6472 &	334.8662	&7.2176
\\\cline{2-6}
        &\makecell{Avg.\\(Std.)} & \makecell{\textbf{6.8874}\\(2.0088)} & \makecell{8.5230\\(1.6610)} &	\makecell{279.2076\\ (102.7063)} &	\makecell{8.2993\\ (1.7031)} \\ \hline
    \end{tabular}
\end{table}
\subsection{Parameter Analysis}
The proposed smoothing model \mbox{\eqref{eq:ms_poisson_aitv}} involves the following parameters:
\begin{itemize}
    \item \textit{The fidelity parameter $\lambda$} weighs how close the approximation $Au^*$ is to the original image $f$. For a larger amount of noise, the value of $\lambda$ should be chosen smaller. 
    \item \textit{The smoothing parameter $\mu$} determines how smooth the solution $u^*$ should be. A larger value of $\mu$ may improve denoising, but at a cost of smearing out the edges between adjacent regions, which may be segmented together if they have similar colors.
    \item \textit{The sparsity parameter $\alpha \in [0,1]$} determines how sparse the gradient vector at each pixel should be. More specifically, the closer the value $\alpha$ is to 1, the more $\|\nabla u\|_1 - \alpha \|\nabla u\|_{2,1}$ resembles $\|\nabla u\|_0$. 
\end{itemize}

We perform sensitivity analysis on these parameters to understand how they affect the segmentation quality of AITV SaT/SLaT. We consider two types of tests in the case of 
 $P/8$ with motion blur in Figure \mbox{\ref{fig:brain}}. In the first case (Figure \mbox{\ref{fig:lambda_analysis}}), we fix $\mu = 1.0$  and vary $\alpha, \lambda$. In the second case (Figure \mbox{\ref{fig:mu_analysis}}),  we fix $\lambda = 5.0$  and vary $\alpha, \mu$.  Figure \mbox{\ref{fig:lambda_analysis}} reveals a concave relationship of the DICE index of each region  with respect to the parameter $\lambda$, which implies there exists the optimal choice of $\lambda.$
Additionally, when $\lambda$ is  small,  a large value for $\alpha$ can improve the DICE indices. According to Figure \mbox{\ref{fig:mu_analysis}}, the DICE indices of the GM and WM regions decrease with respect to $\mu$, while the DICE index of the CSF region is approximately constant. For $\alpha = 0.8$, the DICE indices of the GM and WM regions are the largest when $\mu \geq 1$, but  
the large $\alpha$ is not optimal for CSF. Hence, an intermediate value of $\alpha$, such as $0.6$, is preferable to attain satisfactory segmentation quality for all three regions.

Lastly, in Figure \mbox{\ref{fig:color_analysis}}, we conduct sensitive analysis for the  case of $P=10$ of Figure \mbox{\ref{fig:color_test}}. We fix $\mu = 0.05$, while varying $\lambda, \alpha$ in Figure \mbox{\ref{fig:color_analysis}}(A), which indicates that the optimal value for $\alpha$ is in the range of $0.5 \leq \alpha \leq 0.7$. Then we fix $\lambda = 1.5$ to examine $\mu$ and $\alpha$ in Figure \mbox{\ref{fig:color_analysis}}(B).  For $0.2 \leq \alpha \leq 0.7$, PSNR decreases as $\mu$ increases. Again, $\alpha = 0.6$ generally yields the best PSNR.

\begin{figure}
    \centering
    \includegraphics[scale=0.425]{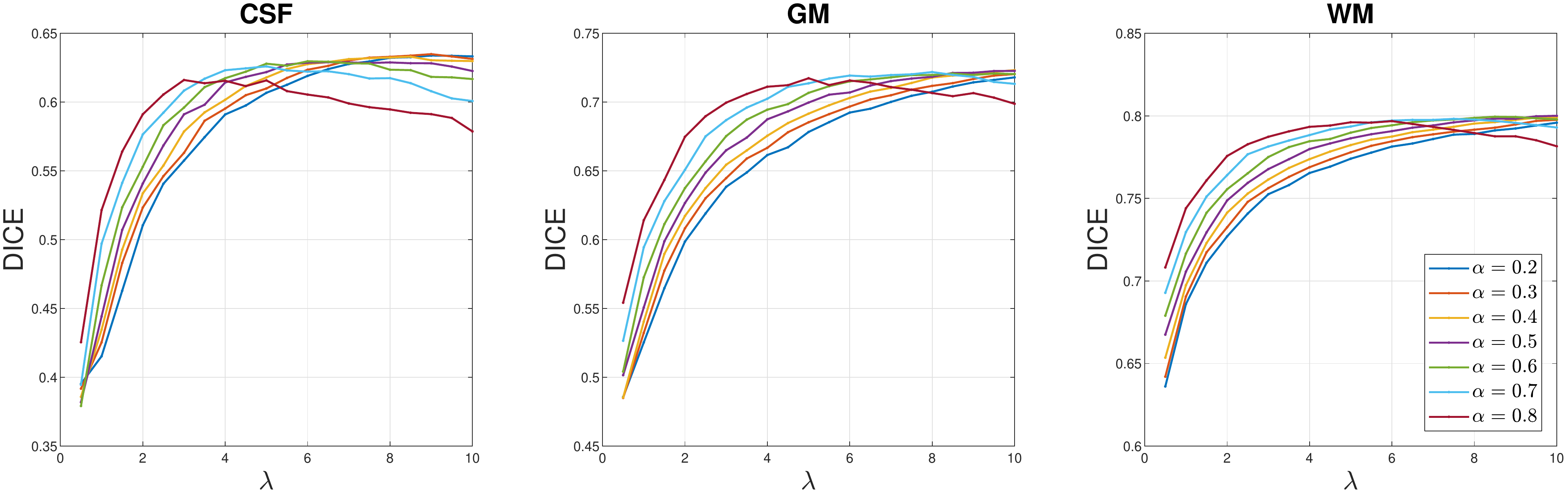}
    \caption{Sensitivity analysis on $\lambda$ for the $P/8$ with motion blur case of Figure \mbox{\ref{fig:brain}}. The parameter $\mu = 1.0$ is fixed.  DICE indices averaged over 10 images for each brain region are plotted with respect to  $\lambda$. }
    \label{fig:lambda_analysis}
\end{figure}
\begin{figure}
    \centering
    \includegraphics[scale=0.425]{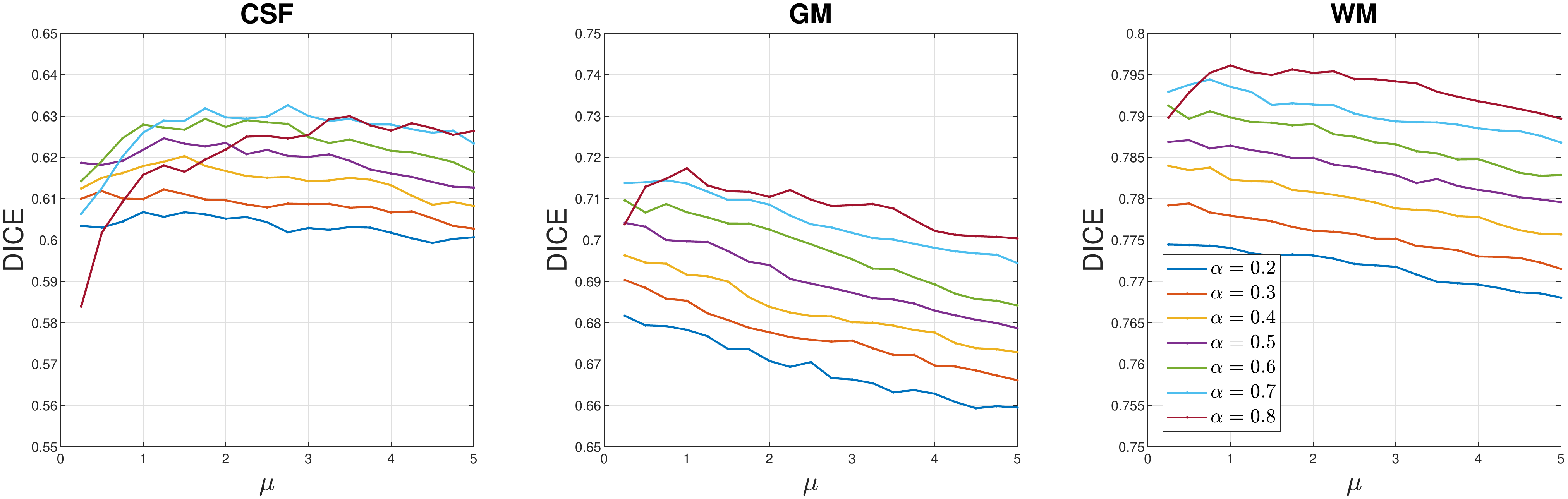}
    \caption{Sensitivity analysis on $\mu$ for the $P/8$ with motion blur case of Figure \mbox{\ref{fig:brain}}. The parameter $\lambda = 5.0$ is fixed.  DICE indices averaged over 10 images for each brain region are plotted with respect to  $\mu$. }
    \label{fig:mu_analysis}
\end{figure}
\begin{figure}
    \centering
    \includegraphics[scale=0.425]{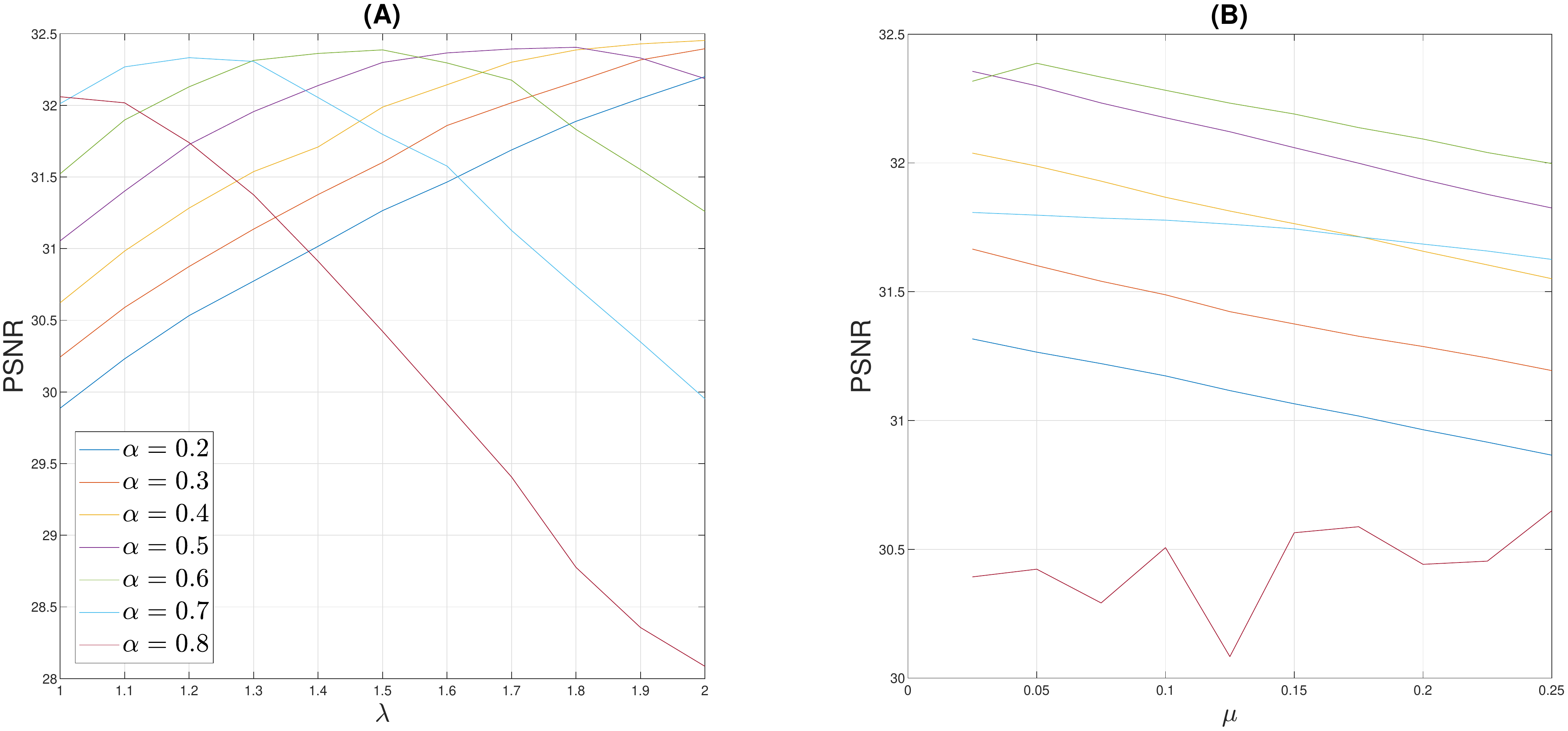}
    \caption{Sensitivity analysis of parameters for $P=10$ case of Figure \mbox{\ref{fig:color_test}}. (A) is the sensitivity analysis on $\lambda$ when $\mu = 0.05$ fixed; (B) is the sensitivity analysis on $\mu$ when $\lambda = 1.5$ fixed. Average PSNR is plotted.}
    \label{fig:color_analysis}
\end{figure}
\section{Conclusion and future work} \label{sec:conclude}
In this paper, we developed the AITV Poisson SaT/SLaT framework for image segmentation. In particular, we proposed a simplified Mumford-Shah model with the AITV regularization and Poisson fidelity for the smoothing step. The model was proven to have a global minimizer. Our numerical algorithm incorporated a specific splitting scheme for ADMM and the $\ell_1 - \alpha \ell_2$ proximal operator for solving a subproblem. Convergence analysis  established that the sequence generated by
ADMM  has a convergent subsequence to a stationary point of the nonconvex model. In our numerical experiments, the AITV Poisson SaT/SLaT yielded high-quality segmentation results within seconds for various grayscale and color images corrupted with Poisson noise and/or blur. For future directions, we are interested in other nonconvex regularization, such as $\ell_1/\ell_2$ on the gradient \cite{wang2022minimizing, wang2021limited, wu2022efficient},  $\ell_p, 0< p<1$, on the gradient \mbox{\cite{hintermuller2013nonconvex, li2020tv, wu2021two}}, and transformed total variation \cite{huo2022stable}, as alternatives to AITV. On the other hand, we can develop AITV variants of weighted TV \mbox{\cite{li2022novel}} or adaptive TV \mbox{\cite{wu2021adaptive, zhang2022edge}}. Moreover, we plan to determine how to make the sparsity parameter $\alpha$ in AITV adaptable to each image. In future work, we will adapt other segmentation algorithms \mbox{\cite{cai2019linkage, jung2017piecewise, jung2014variational,  li2010multiphase, li2016multiphase, pang2023adaptive, yang2022anisotropic}} designed for Gaussian noise or impulsive noise to Poisson noise.

\section*{Conflict of Interest Statement}
The authors declare that the research was conducted in the absence of any commercial or financial relationships that could be construed as a potential conflict of interest.

\section*{Author Contributions}

KB performed the experiments and analysis and drafted the manuscript. All authors
contributed to the design, evaluation, discussions and production of the manuscript. 

\section*{Funding}

The work was partially supported by NSF grants DMS-1846690, DMS-1854434, DMS-1952644, DMS-2151235, and a Qualcomm Faculty Award. 

\section*{Data Availability Statement}
The images in Figure \ref{fig:retinex_test} are provided from the DRIVE dataset \cite{staal2004ridge} at \url{https://drive.grand-challenge.org/DRIVE/}. The images in Figure \ref{fig:brain} are extracted from BrainWeb \cite{aubert2006twenty} via the Python package ``brainweb" provided at \url{https://github.com/casperdcl/brainweb}. The images in Figure \mbox{\ref{fig:color_test}} are selected from the PASCAL VOC 2010 dataset \mbox{\cite{everingham2009pascal}}. Code for AITV Poisson SaT/SLaT is available at \url{https://github.com/kbui1993/Official_Poisson_AITV_SaT_SLaT}.

\bibliographystyle{siam} 
\bibliography{frontiers}




\end{document}